%% file: main.tex
\documentclass[10pt,journal,compsoc]{IEEEtran}
\usepackage{hyperref}       
\usepackage{url}            
\usepackage{booktabs}       
\usepackage{amsfonts}       
\usepackage{nicefrac}       
\usepackage{microtype}      
\usepackage{enumerate}
\usepackage{graphicx}
\usepackage{subfigure}
\usepackage{booktabs}
\usepackage{multirow}
\usepackage{amsmath}
\usepackage{amssymb}
\usepackage{graphicx}
\usepackage{subfigure}
\usepackage{wrapfig}
\usepackage{enumitem}
\usepackage{hyperref}
\usepackage{pifont}
\usepackage{color}
\usepackage[ruled]{algorithm2e}
\usepackage{algorithmic}
\usepackage{amsmath}
\usepackage{amsthm}
\usepackage{bbm}
\usepackage{caption}
\usepackage[switch]{lineno}
\newcommand{\drc}{\textsc{DReg}\xspace}

\newtheorem{theorem}{Theorem}[section]

\usepackage{colortbl} 
\usepackage{makecell} 

\input{math_commands.tex}

\usepackage{setspace}

\usepackage{ragged2e}
\usepackage{floatrow}
\newfloatcommand{capbtabbox}{table}[][\FBwidth]
\usepackage{tikz}

\usepackage{array}
\newcolumntype{L}[1]{>{\raggedright\let\newline\\\arraybackslash\hspace{0pt}}m{#1}}
\usepackage{mathtools}
\usepackage{paralist}
\let\itemize\compactitem
\let\enditemize\endcompactitem
\let\enumerate\compactenum
\let\endenumerate\endcompactenum

\ifCLASSOPTIONcompsoc
  \usepackage[nocompress]{cite}
\else
  \usepackage{cite}
\fi
\ifCLASSINFOpdf
\else
\fi

\begin{document}
\title{Selective Learning: Towards Robust Calibration with Dynamic Regularization}
\author{Zongbo Han, Yifeng Yang, Changqing Zhang, Linjun Zhang, Joey Tianyi Zhou, Qinghua Hu
\IEEEcompsocitemizethanks{
\IEEEcompsocthanksitem Z. Han, Y. Yang, C. Zhang and Q. Hu are with the College of Intelligence and Computing, Tianjin University, Tianjin 300072, China.
\IEEEcompsocthanksitem L. Zhang is with Department of Statistics, Rutgers University, Piscataway, NJ 08854-8019. \IEEEcompsocthanksitem J. T. Zhou is with the A*STAR Centre for Frontier AI Research (CFAR), Singapore 138632.
}
}
\markboth{In submission}%
{Shell \MakeLowercase{\textit{et al.}}: Bare Demo of IEEEtran.cls for Computer Society Journals}

\IEEEtitleabstractindextext{%
\begin{abstract}
\justifying
Miscalibration in deep learning refers to there is a discrepancy between the predicted confidence and performance. This problem usually arises due to the overfitting problem, which is characterized by learning everything presented in the training set, resulting in overconfident predictions during testing. Existing methods typically address overfitting and mitigate the miscalibration by adding a maximum-entropy regularizer to the objective function. The objective can be understood as seeking a model that fits the ground-truth labels by increasing the confidence while also maximizing the entropy of predicted probabilities by decreasing the confidence. However, previous methods lack clear guidance on confidence adjustment, leading to conflicting objectives (increasing but also decreasing confidence). Therefore, we introduce a method called Dynamic Regularization (\drc), which aims to learn what should be learned during training thereby circumventing the confidence adjusting trade-off. At a high level, \drc aims to obtain a more reliable model capable of acknowledging what it knows and does not know. Specifically, \drc effectively fits the labels for in-distribution samples (samples that should be learned) while applying regularization dynamically to samples beyond model capabilities (e.g., outliers), thereby obtaining a robust calibrated model especially on the samples beyond model capabilities. Both theoretical and empirical analyses sufficiently demonstrate the superiority of \drc compared with previous methods.
\end{abstract}
}
\maketitle
\IEEEdisplaynontitleabstractindextext
\IEEEpeerreviewmaketitle

\input{1_introduction}

\input{2_relatedwork}

\input{3_methods}

\subsection{Theory}
\input{theory}

\input{4_experiments}

\section{Conclusion}
In this paper, we summarize the core principle of existing regularization-based calibration methods, and show their underlying limitations due to lack of explicit confidence supervision. To overcome these limitations, we introduce a simple yet effective approach called dynamic regularization calibration (\drc), which regularizes the model using challenging samples that may exceed the capability of the model in the training data, thus allowing us to provide the model with clear direction on confidence calibration by informing the model what it should learn and what it should not learn. \drc significantly outperforms existing methods on real-world datasets, achieving robust calibration performance. Moreover, the theoretical analyses show that \drc achieves a smaller calibration error over previous method. In this work, we first introduce the paradigm of dynamic regularization for calibration and provide a simple yet effective implementation. In the future, we believe exploring more elegant and effective strategies for dynamic regularization will be an interesting and promising direction.
\bibliographystyle{IEEEtran}
\bibliography{main}

\ifCLASSOPTIONcaptionsoff
  \newpage
\fi





\end{document}

%% file: math_commands.tex
\usepackage{amsmath,amsfonts,bm}

\def\eqref#1{equation~\ref{#1}}

\def\1{\bm{1}}

\DeclareMathAlphabet{\mathsfit}{\encodingdefault}{\sfdefault}{m}{sl}
\SetMathAlphabet{\mathsfit}{bold}{\encodingdefault}{\sfdefault}{bx}{n}

\def\Prob{{\mathbb{P}}}

\newcommand{\E}{\mathbb{E}}

\newcommand{\w}{{\bm w}}


%% file: 1_introduction.tex
\section{Introduction}
\IEEEPARstart{D}{}\textcolor{black}{eep neural networks have achieved remarkable progress on various tasks \cite{lecun2015deep}. However, current deep learning classification models often suffer from poor calibration, i.e., their confidence scores fail to accurately reflect the accuracy of the classifier, especially exhibiting excessive confidence during test time \cite{guo2017calibration}. This leads to challenges when deploying them to safety-critical downstream applications like autonomous driving \cite{bojarski2016end} or medical diagnosis \cite{esteva2017dermatologist}, since the decisions cannot be trusted even with high confidences. Therefore, accurately calibrating the confidence of deep classifiers is essential for the successful deployment of neural networks.}
\begin{figure*}[!tbp]
\centering
\includegraphics[width=1\linewidth,height=0.16\linewidth]{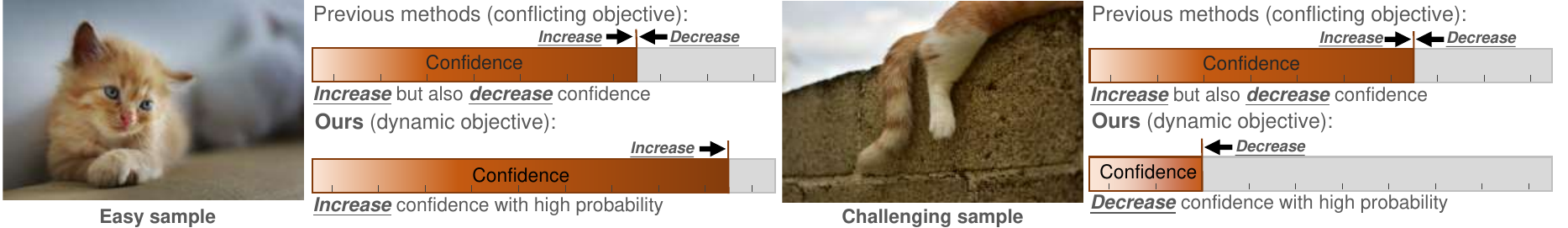}
\caption{\label{fig:motivation} The motivation of the \drc. Previous regularization-based methods aim for accurate classification while maximizing predictive entropy, resulting in conflicting optimization goals of simultaneously increasing and decreasing confidence. To address this issue, \drc dynamically applies regularization to avoid these conflicting optimization objectives, assigning higher confidence to easy samples and lower confidence to challenging samples.}
\end{figure*}
\textcolor{black}{Various regularization-based methods are proposed to avoid overfiting on the training dataset and thereby calibrate deep classifiers not to be overconfident. Specifically, existing empirical evidences reveal that employing strong weight decay and label smoothing can improve calibration performance effectively \cite{guo2017calibration,muller2019does}. More generally, most previous methods directly modify the objective function by implicitly or explicitly adding a maximum-entropy regularizer during training, such as penalizing confidence \cite{pereyra2017regularizing}, focal loss and its extension methods \cite{mukhoti2020calibrating,wang2021rethinking, ghosh2022adafocal, tao2023dual}, and logit normalization \cite{wei2022mitigating}. Specifically, they prevent the model from overfitting on training samples through regularization, thereby alleviating the model's overconfidence problem. The underlying objective of these methods can be intuitively understood as minimizing the classification loss by \textit{increasing confidence} corresponding to the ground-truth label, while simultaneously maximizing the predictive entropy of the predicted probability by \textit{decreasing confidence}. In short, previous methods strive to adjust the confidence while accurately classifying to avoid miscalibration.}

\textcolor{black}{However, existing methods lack clear guidance on when to increase confidence or decrease confidence, which may lead to unreliable prediction when training in practical scenarios where the training set contains both easy samples and challenging samples (e.g., outliers). Specifically, on the one hand, existing methods attempt to apply regularization to prevent the model from overly confident, which may lead to lower confidence even for easy samples that should be classified accurately with high confidence. On the other hand, they still strive to reduce the classification loss by increasing the confidence corresponding to the ground-truth label to fit the labels of these samples, which could lead to higher confidence for challenging samples that are difficult to classify accurately. Therefore, when the training set contains both easy and challenging samples, the potential limitations of existing methods are exposed. Unfortunately, such scenarios are prevalent, arising from various reasons such as differences in inherent sample difficulty \cite{seedat2022data,lorena2019complex, vasudevan2022does, toneva2018empirical}, data augmentation \cite{yun2019cutmix, cubuk2020randaugment} or the existence of multiple subgroups in the data \cite{yao2022improving,han2022umix,yang2023change}. Meanwhile, due to limited model capacity, we may also not be able to accurately classify samples in all cases.}


\textcolor{black}{In summary, the lack of clear supervision on how to adjust confidence may lead to three significant problems of existing methods. (1) They may overly regularize the predictions of model, resulting in higher predicted entropy even for easy samples, potentially leading to an \textit{{under-confident}} model \cite{wang2021rethinking}. In other words, the model might fail to accurately classify samples that it should have been able to recognize effectively. (2) They enforce the model to \textit{{over-confidently}} classify all training samples, even though that may be potentially challenging samples (e.g., outliers) beyond the model capabilities. This will result in the model assigning higher confidence to the classification of samples that are inherently difficult to classify accurately. (3) Striving for accurate classification on the training samples while simultaneously maximizing prediction entropy represents two opposing goals, making it \textit{difficult in balancing} between the two objectives.} 

\textcolor{black}{To this end, we propose a novel method called dynamic regularization (\drc) to solve the above three problems in a unified framework. At a high level, the proposed method is designed to ensure that the model learns selectively during training what it is capable of learning, while acknowledging its limitations in accurately classifying the remaining samples. As illustrated in Fig.~\ref{fig:motivation}, our expectation is that the proposed model will confidently classify samples within its capabilities, while refraining from making decisions on samples that pose a challenge. To achieve this goal, we implicitly construct a probability of whether a sample should be known to the model and then impose dynamic regularization to different samples. Specifically, we avoid imposing regularization to increase predicted entropy for easy samples, which prevents undermining model confidence for easy samples (samples that should be learned). Meanwhile, \drc prevents deep learning models from miscalibration by increasing the predicted entropy on potential outlier samples (samples beyond model capabilities). In this way, \drc can elegantly balance the two opposing goals of accurately classifying training samples and maximizing the entropy of predicted probabilities. By actively distinguishing which sample should be learned, \drc can achieve more robust calibration, especially when the challenging samples are presented in the test set. Both theoretical and experimental results demonstrate the effectiveness of \drc. }

\textcolor{black}{
The contributions of this paper are as follows:
\begin{itemize}
\item We summarize the core principle shared by existing regularization-based calibration methods, and demonstrate the underlying limitations due to their indiscriminate regularization on all samples and lack of clear confidence supervision.
\item We propose a novel paradigm that improves previous regularization-based calibration approaches by informing the model what it should and should not be learned, avoiding problems caused by the lack of explicit confidence guidance in previous works. 
\item We provide theoretical analysis that proves the superiority of \drc in achieving lower calibration error compared to previous regularization-based methods.
\item We conduct extensive experiments on various settings and datasets. The experimental results, with \drc achieving the best performance, strongly suggest \drc outperforms previous approaches in calibration.
\end{itemize}}

%% file: 2_relatedwork.tex
\section{Related Work}

In this section, we first present several closely related research topics, including confidence calibration, uncertainty estimation, and out-of-distribution (OOD) detection. We then summarize how proposed method \drc differs from previous regularization-based approaches in above research topics.

\textbf{Confidence calibration}. A well-calibrated classifier can be approached through two main orthogonal methods: post-hoc calibration and regularization-based calibration. In post-hoc calibration, after the classifier is trained, its predicted confidence is adjusted by training extra parameters on a validation set to improve calibration without modifying the original model \cite{guo2017calibration,yu2022robust}. An example is temperature scaling \cite{guo2017calibration}, where a temperature parameter is trained on the validation set  to scale the predicted probability distribution. On the other hand, regularization-based calibration avoids classifier miscalibration by incorporating regularization techniques during the training of deep neural networks\cite{guo2017calibration, DBLP:conf/icml/KumarSJ18, muller2019does}. This includes strategies such as training with strong weight decay \cite{guo2017calibration}, label smoothing \cite{muller2019does}, penalizing confidence \cite{pereyra2017regularizing}, focusing on under-confident samples \cite{mukhoti2020calibrating, ghosh2022adafocal, wang2021rethinking, tao2023dual}, and constraining the norm of logits \cite{wei2022mitigating,wei2023mitigating}. Note that among the two main orthogonal methods of post-hoc calibration and regularization-based calibration, the main research focus of this paper is on the latter. 

\textbf{Uncertainty estimation}. The primary goal of uncertainty estimation is to obtain the reliability of the model. Traditional uncertainty estimation methods utilize ensemble learning \cite{lakshminarayanan2017simple,liu2019accurate} and Bayesian neural networks \cite{neal2012bayesian,mackay1992practical,denker1990transforming,kendall2017uncertainties} to obtain the distribution of predictions, from which uncertainty can be estimated using the variance or entropy of the distribution. Recently, several regularization-based uncertainty estimation methods have been proposed. These methods typically obtain uncertainty by applying regularization to neural networks using either additional out-of-distribution dataset \cite{malinin2018predictive} or training set samples \cite{sensoy2018evidential}.

\textbf{Out-of-distribution (OOD) detection}. OOD detection aims to distinguish potential out-of-distribution samples during test time to avoid unreliable predictions. One of the primary approaches is to modify the loss function during training, thereby constraining the deep neural network to effectively identify potential OOD samples. Specifically, existing methods usually impose constraints on the model, such as having a uniform prediction distribution \cite{lee2017training,hendrycks2018deep,choi2023conservative} or higher energy \cite{katz2022training,liu2020energy,du2021vos,chen2021atom}, particularly on the additional outlier training set.

\textbf{Comparison with existing methods}. The core principle shared by regularization-based approaches in above topics is to impose regularization on neural networks during training to prevent overfitting. These methods can be roughly categorized into two types: (1) \textit{standard training without additional training data} -- These methods aim to simultaneously achieve accurate classification and regularize the model to maximize prediction entropy on the training dataset. However, as shown in the introduction, there is an underlying trade-off between achieving high accuracy and high entropy, making it  challenging to effectively balance these two objectives. (2) \textit{training utilizing additional outlier data} -- Methods along this line leverage external additional outlier datasets to regularize the model. The prohibitively large sample space of  outliers requires additional sample selection strategies to obtain outliers that can effectively regularize the  behavior of the model \cite{ming2022poem}. However, the artificially designed sample selection strategies may introduce biases into the distribution of outlier samples, resulting in a failure to characterize the true distribution of OOD data encountered at test time. Besides, how to effectively utilize outlier samples to improve calibration is still an open problem. In contrast to the aforementioned methods, \drc aims to explore and leverage information from naturally occurring challenging samples within the training set that have a high probability of being similar to ones encountered at test time. In this way, \drc eliminates the need for additional external outlier datasets and the corresponding sample selection strategies, while avoiding the trade-off between model performance and regularization \drc.

%% file: 3_methods.tex
\section{Dynamic Regularization for Calibration}
Previous regularization-based calibration methods typically face challenges due to conflicting objectives: achieving accurate classification while maximizing predictive entropy. To address this challenge, we aim to identify and learn from samples that the model can accurately classify during training, while concurrently avoiding decisions on potentially challenging samples that present a high risk of misclassification. We start by introducing the problem setting of neural network calibration, then analyze existing regularization methods in neural networks for calibration, and finally present our proposed approach.
\subsection{Problem Setting}
\textbf{Notations}. Let $\mathcal{X}$ and $\mathcal{Y}=\{1,2,\cdots,K\}$ be the input and label space respectively, where $K$ denotes the number of classes. The training dataset $\mathcal{D}=\{\boldsymbol{x}_i,y_i\}_{i=1}^n$ consists of $n$ samples independently drawn from a training distribution $P$ over $\mathcal{X}\times \mathcal{Y}$. The label $y$ can also be represented as a one-hot vector $\boldsymbol{y}=[y^1,\cdots,y^K]$, where $y^k = \mathbbm{1}(k=y)$ and $\mathbbm{1}(\cdot)$ is the indicator function. The goal of the classification task is to train a model $f_{\boldsymbol{\theta}}: \mathcal{X}\rightarrow \Delta^{K}$ parameterized with $\boldsymbol{\theta} \in \Theta$, where $\Delta^{K}$ denotes  the classification probability space. This model can produce a predictive probability vector $f_{\boldsymbol{\theta}}(\boldsymbol{x}) = [f_{\boldsymbol{\theta}}^1(\boldsymbol{x}),\cdots,f_{\boldsymbol{\theta}}^K(\boldsymbol{x})]$ for input $\boldsymbol{x}$. The predicted label and the corresponding confidence score can be obtained with $\hat{y}:= \arg \max f_{\boldsymbol{\theta}}(\boldsymbol{x})$ and $\hat{\Prob}(\boldsymbol{x}):=\max f_{\theta}(\boldsymbol{x})$, respectively. Moreover, $\hat \Prob(\boldsymbol{x})$ is generally considered as the estimated probability that $\hat{y}$ is the correct label for $\boldsymbol{x}$.

\textbf{Confidence calibration}. A model is considered perfectly calibrated \cite{guo2017calibration} when the estimated probability or confidence score $\hat{\Prob}(\boldsymbol{x})$ precisely matches the probability of the model classifying $\boldsymbol{x}$ correctly, i.e., $\Prob(\hat{y}=y\mid \hat{\Prob}(\boldsymbol{x})=p)=p$ for all $p \in [0,1]$. For instance, for samples that obtain an estimated confidence of $0.8$, a perfectly calibrated neural network should have an accuracy of 80\%. However, recent empirical observations have revealed that widely used deep neural networks often exhibit overconfidence due to overfitting on the training data \cite{guo2017calibration,mukhoti2020calibrating}. The objective of confidence calibration is to obtain a model whose predicted confidence scores can accurately reflect the predictive accuracy of the model. 

\subsection{Regularization for Calibration}
Applying regularization on the model can prevent overfitting on the training dataset, thus alleviating overconfidence. 
Existing regularization-based methods typically incorporate implicit or explicit regularization on the model by modifying the objective function. Here, we first present four representative methods and then show that they share similar goals and limitations.

\textbf{Label smoothing} is a regularization technique that involves training the neural network with softened target labels $\boldsymbol{\widetilde{y}}=[\widetilde{y}^1,\cdots,\widetilde{y}^K]$ to improve the reliability of confidence \cite{muller2019does}, where $\widetilde{y}^k=(1-\epsilon)y^k+\epsilon /K$ and $0<\epsilon<1$ controls the strength of smoothing. This approach effectively prevents overfitting in deep neural networks and improves the reliability of confidence scores \cite{muller2019does}. 
When using cross-entropy loss, the label smoothing loss can be decomposed as follows:
\begin{equation}
\label{eq:ls}
\mathcal{L}_{ce}(f_{\boldsymbol{\theta}}(\boldsymbol{x}),\boldsymbol{\widetilde{y}})=(1-\epsilon)\mathcal{L}_{ce}(f_{\boldsymbol{\theta}}(\boldsymbol{x}),\boldsymbol{y}) + \epsilon \mathcal{L}_{ce}(f_{\boldsymbol{\theta}}(\boldsymbol{x}), \boldsymbol{p}_{u}),
\end{equation}
where $\mathcal{L}_{ce}$ denotes the cross-entropy loss and $\boldsymbol{p}_{u}=[1/K,\cdots,1/K]$ is a uniform distribution.

\textbf{Focal loss} and other related calibration methods implicitly regularize the deep neural network by increasing the weight of samples with larger losses. This prevents the model from overfitting due to excessive attention to simple samples \cite{mukhoti2020calibrating,ghosh2022adafocal,tao2023dual}. The classical focal loss is defined as $\mathcal{L}_f(f_{\boldsymbol{\theta}}(\boldsymbol{x}), \boldsymbol{y})=-(1-f_{\theta}^y(\boldsymbol{x}))^\gamma \log f_{\theta}^y(\boldsymbol{x})$, where $\gamma$ is a hyperparameter that controls the reweighting strength, allowing higher weight to be assigned to samples with lower confidence in the correct label. A general form of focal loss can be formulated as an upper bound of the regularized cross-entropy loss  \cite{mukhoti2020calibrating}, 
\begin{equation}
\label{eq:focal}
\mathcal{L}_f(f_{\boldsymbol{\theta}}(\boldsymbol{x}), \boldsymbol{y})\geq \mathcal{L}_{ce}(f_{\boldsymbol{\theta}}(\boldsymbol{x}), \boldsymbol{y})-\gamma \mathcal{H}(f_{\theta}(\boldsymbol{x})),
\end{equation}
where $\mathcal{H}(f_{\theta}(\boldsymbol{x}))$ represents the entropy of the predicted distribution.

\textbf{Evidential deep learning} \cite{sensoy2018evidential} constructs the classification outputs as Dirichlet distributions $Dir(\boldsymbol{\alpha})$ with parameters $\boldsymbol{\alpha}=[\alpha^1,\cdots,\alpha^K]$, and then minimizes the expected distance between the obtained Dirichlet distributions and the labels. Meanwhile, the regularization is conducted by minimizing the KL divergence between the obtained Dirichlet distributions and the uniform distribution. The loss function is formulated as follows:
\begin{equation}
\label{eq:evi}
\sum_{k=1}^{K} y^{k}\left(\psi(S)-\psi(\alpha^{k})\right) + \gamma KL[Dir(\tilde{\boldsymbol{\alpha}}), Dir([1,\cdots, 1])],
\end{equation}
where $S=\sum_{k=1}^K \alpha^k$ is the Dirichlet strengthes, $\psi(\cdot)$ represents the digamma function, $\gamma$ is a hyperparameter, $\tilde{\boldsymbol{\alpha}} = \boldsymbol{y} + (1-\boldsymbol{y})\odot\boldsymbol{\alpha}$, and $Dir([1,\cdots, 1])$ is a uniform Dirichlet distribution. 

\textbf{Penalizing confidence} \cite{pereyra2017regularizing} suggest a confidence penalty term to prevent the deep neural networks from overfitting and producing overconfident predictions. Formally, the loss function of penalizing confidence is defined as:
\begin{equation}
\label{eq:pc}
\mathcal{L}_{ce}(f_\theta(\boldsymbol{x}),\boldsymbol{y}) - \gamma \mathcal{H}(f_\theta(\boldsymbol{x})),
\end{equation}
where $\gamma$ is a hyperparameter to control the penalizing strength.

It can be seen from Eq.~\ref{eq:ls}, Eq.~\ref{eq:focal}, Eq.~\ref{eq:evi} and Eq.~\ref{eq:pc} that the objective of regularization-based calibration algorithms involves minimizing the classification loss while maximizing the predictive entropy. Specifically, the first objective typically represents a common classification loss (e.g., cross-entropy loss) that minimizes the discrepancy between $f_{\boldsymbol{\theta}}(\boldsymbol{x})$ and $\boldsymbol{y}$. Its goal is to enable the model to accurately classify training set samples with high confidence.
The second objective serves as the regularization term, which prevents overconfidence of the deep neural network by encouraging lower confidence. For instance, in Eq.~\ref{eq:ls}, the second term aims to push the predicted distribution closer to the uniform distribution, while in Eq.~\ref{eq:focal}, it aims to maximize the entropy of the predicted distribution. In this way, the model is able to reject classification when the test samples cannot be accurately classified.

However, the above objectives faces a conflict, i.e., minimizing the classification loss by \textit{increasing the confidence} corresponding to the ground-truth label and maximizing the entropy of the predictive probability by \textit{decreasing the confidence}. To increase and decrease confidence are two opposite goals and striking the appropriate balance between them is difficult, leading to potential problems. For example, the model obtained with weak-regularization may classify all samples with high confidence, leading to overconfident especially on the challenging samples. Besides, strong regularization may try to reduce confidence even on the easy samples, which could lead to an underconfident model especially on the easy samples, even harming the overall model performance \cite{wang2021rethinking}. The key reason for these problems with previous methods is that they apply the same regularization to all samples without considering the inherent difficulty in correctly classifying each training sample. 

\subsection{Dynamic Regularization}
In this section, we propose a novel approach to overcome the limitations observed in existing regularization methods. At a high level, our method aims to learn samples that can be classified with high confidence when within the competency of model. Meanwhile, it is able to reject classifications for samples that fall outside the model's capability according to the predictive confidence. To achieve these goals, we first propose to model the data distribution of simple samples within the capability of the model and difficult samples outside the capability. Then we leverage the inherent difficult samples (e.g., outliers) in the training set to provide clearer guidance on how the model should adjust its predictive confidence on each sample with probabilistic modeling. Intuitively, we recognize that strong regularization is not required for easy samples, which should be classified correctly with higher confidence. Conversely, for challenging samples beyond the model capability, a lower confidence score should be assigned. By applying the above dynamic regularization, the problem of balancing conflicting objectives in the previous methods can be solved gracefully.

To identify what should be learned in the training set, we first make the following assumption that the training set distribution $P$ follows the Huber's $\eta$-contamination model \cite{huber1992robust}, i.e.,
\begin{equation}\label{eq:huber}
P = (1-\eta)P_{in} + \eta P_{out},
\end{equation}
where $P_{in}$ represents the data distribution that falls within the capability of the model, e.g., in-distributional data. $P_{out}$ denotes the data distribution that is outside the capability of the model, e.g., arbitrary outlier distribution. $\eta$ represents the fraction of the challenging data in the training dataset. With a fixed model and dataset, $\eta$ is a latent unknown parameter and can be treated as an adjustable hyperparameter in practice. In practical scenarios, training data often contains a mix of simple and challenging samples, making the aforementioned assumption about the training distribution reasonable. Specifically, these simple samples (e.g., in-distribution data) are more likely to be classified correctly, while some difficult samples may not be classified correctly owing to the lack of crucial features or inherent ambiguity and indistinguishability \cite{seedat2022data,lorena2019complex}. Empirical studies have also shown that even widely used datasets like ImageNet, CIFAR10, and CIFAR100 contain many easily classifiable samples as well as outlier samples that pose challenges for classification \cite{vasudevan2022does,toneva2018empirical}. Besides, the data augmentation methods, a widely-used component in the training of deep neural network, create novel samples by perturbing the training data, potentially introducing outlier samples \cite{yun2019cutmix,cubuk2020randaugment}. Considering the existence of challenging samples  in the training data, we expect the model to be able to accurately classify samples within its capability while having robust calibration performance for samples beyond the capability of the model. 
To this end, we propose the following dynamic regularization for robust confidence calibration.
\definecolor{mygreen}{RGB}{0,128,128}
\begin{algorithm*}[!htbp]
    \caption{The training pseudocode of the proposed method \drc.\label{alg:train}}
    \KwIn{
        Training dataset $\mathcal{D}$, hyperparameter of challenging data fraction $\eta$ and regularization strength $\beta$\;
    }
    \KwOut{
        The trained neural network $f_{\boldsymbol{\theta}}$;
    }
    \For{each iteration}{
     \textcolor{mygreen}{// Randomly sample $B$ training samples from training set.}\\
    Sample $B$ training samples $\{\boldsymbol{x}_i, \boldsymbol{y}_i\}_{i\in\mathbb{B}}$ from the training set $\mathcal{D}$ with a random index set $\mathbb{B}$;\\ 
     \textcolor{mygreen}{// Compute the classification loss to estimate the difficulty of classifying each sample.}\\    
    Compute the corresponding classification loss (e.g., $\mathcal{L}_{ce}(f_{\boldsymbol{\theta}}(\boldsymbol{x}_i),\boldsymbol{y}_i)$) of each sample;\\
     \textcolor{mygreen}{// Sort the samples by the difficulty of classifying them to identify challenging samples.}\\
    Sort the losses and set $\delta_i$ for each sample according to the sorting result;\\
    \textcolor{mygreen}{// Apply dynamic regularization with the implementation of the proposed method (Eq.~\ref{eq:loss}).}\\
    Compute the loss $\widetilde{\mathcal{L}}_{\drc}(f_{\boldsymbol{\theta}}(\boldsymbol{x}_i),\boldsymbol{y}_i)$ according to Eq.~\ref{eq:loss} for each sample;\\
    Update $\boldsymbol{\theta}$ by one step to minimize $\mathbb{E}_{i\in\mathbb{B}}[\widetilde{\mathcal{L}}_{\drc}(f_{\boldsymbol{\theta}}(\boldsymbol{x}_i),\boldsymbol{y}_i)]$ with some gradient method.
    }
\end{algorithm*}

Unlike previous methods that regularize all samples indiscriminately, we impose dynamic regularization on the samples by considering the existence of challenging samples to provide explicit confidence supervision. Specifically, we first introduce the probability $P_{\boldsymbol{x}\sim P_{in}}$ and $P_{\boldsymbol{x}\sim P_{out}}$ that a training sample $\boldsymbol{x}$ belongs to the $P_{in}$  or $P_{out}$, where $P_{\boldsymbol{x}\sim P_{out}}+P_{\boldsymbol{x}\sim P_{in}}=1$. Then we formally define the following probabilistical dynamic regularization objective function $\mathcal{L}_{\drc}$:
\begin{equation}
\begin{aligned}
\label{eq:orc}
\mathcal{L}_{\drc}&(f_{\boldsymbol{\theta}}(\boldsymbol{x}),\boldsymbol{y}) =
\\& P_{\boldsymbol{x}\sim P_{in}} \mathcal{L}_{in}(f_{\boldsymbol{\theta}}(\boldsymbol{x}),\boldsymbol{y}) + P_{\boldsymbol{x}\sim P_{out}} \mathcal{L}_{out}(f_{\boldsymbol{\theta}}(\boldsymbol{x}),\boldsymbol{y}),
\end{aligned}
\end{equation}
where $\mathcal{L}_{in}$ represents the classification loss for the samples within the capability of model, and in practice we can directly utilize the cross-entropy loss or square loss. $\mathcal{L}_{out}$ is the loss function for samples beyond the capability of the model, acting as a regularizer to encourage lower confidence in these samples. Eq.~\ref{eq:orc} has an intuitive motivation that model should admit some samples are within its ability and should be classified accurately with high confidence, while others are outside its ability. Specifically, for samples with a high probability of being in-distribution data (larger $P_{\boldsymbol{x}\sim P_{in}}$), we should minimize its classification loss to increase the confidence. Meanwhile, for samples with a high probability of out of the model capability (larger $P_{\boldsymbol{x}\sim P_{out}}$), we should regularize its prediction confidence to avoid making overconfident decisions. By applying the dynamic regularization, we can achieve confident predictions on samples that the model should classify accurately, while avoiding overconfidence on challenging samples.

However directly estimating the probability $P_{\boldsymbol{x}\sim P_{in}}$ and $P_{\boldsymbol{x}\sim P_{out}}$ presented in Eq.~\ref{eq:orc} is intractable since which distribution the training example is from is unknown. Therefore, we first present a simplified implementation of Eq.~\ref{eq:orc} and then show that this simplified version implicitly estimates $P_{\boldsymbol{x}\sim P_{in}}$ and $P_{\boldsymbol{x}\sim P_{out}}$ to achieve the same effect as the original objective when considering the whole training process. Specifically, in practice, given $B$ training examples during each training step, the simplified dynamic regularization calibration loss can be formulated as follows:
\begin{equation}
\begin{aligned}
\label{eq:loss}
\widetilde{\mathcal{L}}_{\drc}&(f_{\boldsymbol{\theta}}(\boldsymbol{x}_i),\boldsymbol{y}_i) = 
\\&\delta_{i} \mathcal{L}_{in}(f_{\boldsymbol{\theta}}(\boldsymbol{x}_i),\boldsymbol{y}_i) + (1-\delta_{i}) \mathcal{L}_{out}(f_{\boldsymbol{\theta}}(\boldsymbol{x}_i), \boldsymbol{y}_{i}).
\end{aligned}
\end{equation}
$\delta_i$ can be seen as whether the sample is within the capability of the model and it is set to $0$ if classification loss of $\boldsymbol{x}_i$ ranks in the top-$\eta B$ in a batch of $B$ samples, otherwise it is set to $1$ , where $\eta$ is a proportional hyperparameter about how many samples in the training set are within the capabilities of the model. 
Eq.~\ref{eq:loss} simplifies the concept presented in Eq.~\ref{eq:orc} by performing non-parametric binary classification to avoid directly estimating the intractable $P_{\boldsymbol{x}\sim P_{in}}$ and $P_{\boldsymbol{x}\sim P_{out}}$. But when considering the whole training process, $P_{\boldsymbol{x}\sim P_{in}}$ and $P_{\boldsymbol{x}\sim P_{out}}$ can be implicitly estimated and Eq.~\ref{eq:loss} can have the similar objective to Eq.~\ref{eq:orc}. Specifically, if sample $\boldsymbol{x}_i$ is sampled $S$ times during the whole training process, and at the $s$-th sampling $\boldsymbol{x}_i$ is categorized as sample out of the capability of the model ($\delta_i^s=0$) or sample within the capability of the model ($\delta_i^s=1$). Then after $S$ times samplings, the whole objective for $\boldsymbol{x}_i$ can be written as 
\begin{equation}
\begin{aligned}
\label{eq:Sloss}
\begin{matrix}
\sum_{s=1}^S
\end{matrix}[\delta_i^s \mathcal{L}_{in}(f_{\boldsymbol{\theta}}(\boldsymbol{x}_i),\boldsymbol{y}_i) + (1-\delta_i^s) \mathcal{L}_{out}(f_{\boldsymbol{\theta}}(\boldsymbol{x}_i),\boldsymbol{y}_{i})].
\end{aligned}
\end{equation}
Then $P_{\boldsymbol{x}\sim P_{in}}$ and $P_{\boldsymbol{x}\sim P_{out}}$ are implicitly estimated as $\sum_{s=1}^S \delta_i^s/S$ and $1-\sum_{s=1}^S \delta_i^s/S$ with multiple samplings respectively. 
In practice, $\mathcal{L}_{in}$ can be regarded as a classification loss, and we can employ standard cross-entropy loss $\mathcal{L}_{ce}$ or mean square error loss, etc. In this paper we directly employ cross-entropy loss $\mathcal{L}_{ce}(f_{\boldsymbol{\theta}}(\boldsymbol{x}_i),\boldsymbol{y}_i)$ serving as loss function of potential in-distribution data. $\mathcal{L}_{out}$ can be regarded as a regularization loss and we use the weighted KL divergence $\beta \mathcal{L}_{kl}(f_{\boldsymbol{\theta}}(\boldsymbol{x}_i), \boldsymbol{p}_{u})$ between the predicted probability $f_{\boldsymbol{\theta}}(\boldsymbol{x}_i)$ and uniform distribution $\boldsymbol{p}_{u}$ , where $\beta$ is a hyperparameter to tune the strength of the KL-divergence. 
The whole training pseudocode of the \drc is shown in Alg.~\ref{alg:train}.

%% file: theory.tex
In this section, we aim to characterize the calibration error of \drc and the previous regularization-based methods under the Huber’s $\eta$-contamination model (Eq.~\ref{eq:huber}). The results demonstrate that \drc could obtain smaller calibration error.

\textbf{Data generative model.} We consider the Huber’s $\eta$-contamination model described in Eq.~\ref{eq:huber}. Specifically, we assume $P_{in}$ follows a Gaussian mixture model for binary classification, where $X$ is standard Gaussian, $Y\in\{-1,1\}$ with prior probability $\Prob(Y=1)=\Prob(Y=-1)$, and
$$
X\mid Y\sim N(Y\cdot \boldsymbol{w}^*,I_d),
$$
where $I_d$ denotes the $d$-dimensional identity matrix, $\boldsymbol{w}^*$ is the ground-truth coefficient vector. We further assume $P_{out}$ follows the opposite distribution where $X\mid Y\sim N(-Y\cdot \boldsymbol{w}^*,I_d)$. We have $i.i.d.$ observations $\{(\boldsymbol{x}_i,{y}_i)\}_{i=1}^n$ sampled from the distribution $P = (1-\eta)P_{in} + \eta P_{out}$. 

\textbf{The baseline estimator.} For simplicity, we consider the method of minimizing the common used classification loss as shown in the Eq.~\ref{eq:clsloss} with label smoothing as our baseline, which produces a solution:
\begin{equation}
\label{eq:clsloss}
\hat{\boldsymbol{w}}=\arg\min_{\boldsymbol{w}}\frac{1}{n}\sum_{i=1}^n(\boldsymbol{w}^\top \boldsymbol{x}_i-\tilde y_i)^2,
\end{equation}
where $\tilde y_i=(1-\epsilon)y_i+\epsilon/2$. 
For $k\in \{-1,1\}$, the confidence $\hat{\mathbb{P}}_k(\boldsymbol{x})$ is an estimator of $\hat{\mathbb{P}}(y=k|\boldsymbol{x})$, and it takes the form $\label{eq:score} \hat{\mathbb{P}}_k(\boldsymbol{x})={1}/(e^{-k\cdot\hat{\boldsymbol{w}}^\top \boldsymbol{x}}+1)$. Other regularization-based methods essentially have similar goals to label smoothing.

\textbf{Calibration error.} Here we consider the case where $\Prob_1(X)>1/2$, as the case where $\Prob_1(X)\le 1/2$ can be analyzed similarly by symmetry. For $p\in(1/2,1)$, the signed calibration error at a confidence level $p$ is $
p-\mathbb{P}(Y=1\mid \hat{\mathbb{P}}_1(X)=p).
$
As a result, the formula of calibration error (ECE) is given by 
$$\textsl{\textrm{ECE}}[\hat{\mathbb{P}}] = \mathbb{E}[|\mathbb{P}(Y | \hat{\mathbb{P}}(X) = p) - p|].
    $$
In the following, we show that the calibration error of our proposed Algorithm~\ref{alg:train}, denoted by $\hat{\mathbb{P}}_{\drc}$, is smaller than the baseline algorithm $\hat{\mathbb{P}}_{baseline}$. 

\begin{theorem}\label{thm:calibration}
Consider the data generative model and the learning setting above.  We assume $\|\boldsymbol{w}^*\|\le c_0$ for some sufficiently small $c_0>0$, and $d/n=o(1)$. Suppose the initialization parameter $\boldsymbol{\theta}^{(0)}$ satisfies $\|\boldsymbol{\theta}^{(0)}-\boldsymbol{w}^*\|\le c_1$ for a sufficiently small constant $c_1>0$. Then, for sufficiently large $n$, for $k=2,\ldots,K$, we have
$$
\textrm{ECE}[\hat{\mathbb{P}}_{\drc}]< \textrm{ECE}[\hat{\mathbb{P}}_{baseline}].
$$
\end{theorem}

\begin{proof} Proof of Theorem~\ref{thm:calibration}. Following \cite{bai2021don}, we have 
\begin{equation}
\begin{aligned}
p-&\Prob(Y=1\mid \hat{\mathbb{P}}_1(X)=p)=\\&
p-\E_Z[\sigma(\frac{\|\w^*\|}{\|\hat\w\|}\cos\hat\theta\cdot\sigma^{-1}(p))+\sin\hat\theta\cdot\|\w^*\|Z],   
\end{aligned}
\end{equation}
where $\cos\hat\theta=\frac{\hat\w^\top\w^*}{\|\hat\w\|\cdot\|\w^*\|}.$

We first compute the calibration error for the baseline method $$
\hat{\boldsymbol{w}}=\arg\min_{\boldsymbol{w}}\frac{1}{n}\sum_{i=1}^n(\boldsymbol{w}^\top \boldsymbol{x}_i-\tilde y_i)^2,
$$
where $\tilde y_i=(1-\epsilon)y_i+\epsilon/2$. 
As ${d/n}= o(1)$, we have 
\begin{equation}
\begin{aligned}
\hat\w= (\E[x_i&x_i^\top])^{-1} \E[x_i \tilde y_i]+o(1) \\&
=\frac{1-\epsilon}{1+\|\w^*\|^2} \E[x_i  y_i]+o(1)\\&
=(1-\epsilon)(1-2\eta)\cdot\w^*+o(1).
\end{aligned}
\end{equation}
In the above derivation, the first equation uses Sherman–Morrison formula.

As a result, we have $\cos\hat\theta= 1-o(1)$ as $n$ grows, and therefore when $n\to\infty$,
\begin{equation}
\begin{aligned}
p-\Prob(Y=1\mid \hat{\mathbb{P}}&_{baseline}(X)=p)=\\&
p-\sigma(\frac{1}{(1-\epsilon)(1-2\eta)}\sigma^{-1}(p)).
\end{aligned}
\end{equation}
Now, for the \drc, when the initialization parameter $\boldsymbol{\theta}^{(0)}$ satisfies $\|\boldsymbol{\theta}^{(0)}-\boldsymbol{w}^*\|\le c_1$ for a sufficiently small constant $c_1>0$, there will be only $o(1)$ outliers left, and therefore
\begin{equation}
\begin{aligned}
p-\Prob(Y=1\mid \hat{\mathbb{P}}_{\drc}(X)=p)=p-\sigma(\frac{1}{1-\eta}\sigma^{-1}(p)).
\end{aligned}
\end{equation}
By the monotonicity of $\sigma$ and the nonnegativity of $\sigma^{-1}(p)$ when $p>1/2$. We have 
\begin{equation}
\begin{aligned}
|p-\Prob&(Y=1\mid \hat{\mathbb{P}}_{\drc}(X)=p)|\\&<|p-\Prob(Y=1\mid \hat{\mathbb{P}}_{baseline}(X)=p)|.    
\end{aligned}
\end{equation}
Taking the expectation of $p$ for both sides, we have 
$$
\textsl{\textrm{ECE}}[\hat{\mathbb{P}}_{\drc}]< \textsl{\textrm{ECE}}[\hat{\mathbb{P}}_{baseline}].
$$
\end{proof}

%% file: 4_experiments.tex
\section{Experiments}
We conduct extensive experiments on multiple datasets with challenging examples in the training or test dataset to answer the following questions. Q1 Effectiveness: Does \drc outperform other methods in terms of accuracy?  Q2 Reliability: Can \drc obtain a more reliable model? Q3 Robustness: How does \drc perform on  data that may outside the capability of the model? Q4 Ablation study: How would the performance be if challenging samples in the training dataset are not exploited? Q5 Hyperparameter analysis: How do the hyperparameters in \drc affect model performance? 
\subsection{Experimental Setup}
In this section, we describe the experimental setup including the used experimental datasets, evaluation metrics, experimental settings and comparison methods. 

\newcommand{\mycustomsize}{\fontsize{8.0}{8.3}\selectfont}
\begin{table*}[!htbp]
\mycustomsize
  \centering
  \caption{The comparison experimental results on different datasets and different methods under the standard setting. ↓ and ↑ indicate lower and higher values are better respectively. For better presentation, the best and second-best results are in {\textbf{bold}} and \underline{uderline} respectively. For clarity, NLL values are multiplied by 10. Remaining values are reported as percentages (\%). For datasets with both in-distributional (ID) easy samples and hard test samples, we report the results on all samples and hard samples respectively. We mark whether the test samples are sampled from ID or hard data in the table. \textit{Compared with other methods, \drc achieves excellent performance on different metrics in almost all datasets.}}
    \begin{tabular}{cc|c|cccc|ccc}
    \toprule
    \multirow{2}[2]{*}{\makecell{Dataset}} & \multicolumn{1}{c|}{\multirow{2}[2]{*}{Method}} & \multicolumn{1}{c|}{ACC} & \multicolumn{1}{c}{EAURC} & \multicolumn{1}{c}{AURC} & \multicolumn{1}{c}{AUPR} & \multicolumn{1}{c|}{FPR95\%}  & \multicolumn{1}{c}{ECE} & \multicolumn{1}{c}{Brier} & \multicolumn{1}{c}{NLL} \\
    & \multicolumn{1}{c|}{} & \multicolumn{1}{c|}{(↑)} & (↓)   & \multicolumn{1}{c}{(↓)} & Err (↑) & \multicolumn{1}{c|}{TPR(↓)} & (↓)   & (↓)   & (↓) \\
    \midrule
    \multicolumn{1}{c}{\multirow{9}[0]{*}{\makecell{CIFAR\\-8-2\\(All)}}} & ERM   & 77.86±0.19 & 1.84±0.66 & 4.50±0.71 & 74.27±5.70 & 41.81±14.15 & 14.17±5.27 & 35.22±5.55 & 16.06±6.34 \\
          & PC    & 77.36±0.14 & 2.67±1.08 & 5.45±1.09 & 71.12±6.43 & 48.10±12.39 & 18.06±2.45 & 39.34±3.32 & 23.63±9.21 \\
          & LS    & 77.93±0.23 & 5.01±2.95 & 7.65±2.99 & 74.82±6.05 & 34.36±9.55 & 8.37±1.02 & 30.08±2.61 & \underline{7.66±0.65} \\
          & FLSD  & 76.79±0.12 & 2.96±0.16 & 5.89±0.18 & 68.88±1.13 & 56.54±1.52 & 11.63±0.07 & 35.42±0.26 & 11.87±0.15 \\
          & FL    & 77.26±0.25 & 2.37±0.05 & 5.18±0.12 & 70.58±0.18 & 52.30±0.88 & 16.31±0.13 & 37.65±0.30 & 15.38±0.21 \\
          & IFL   & \underline{78.30±0.18} & \underline{1.02±0.14} & \underline{3.57±0.10} & \underline{82.98±1.27} & {\textbf{18.80±2.05}} & \underline{5.64±1.04} & \underline{26.80±0.76} & 7.92±0.88 \\
          & DFL   & 77.14±0.16 & 2.35±0.07 & 5.19±0.05 & 69.92±1.06 & 53.53±1.44 & 16.00±0.09 & 37.65±0.24 & 15.48±0.07 \\
          & RC    & 77.57±0.31 & 4.10±0.44 & 6.83±0.40 & 65.33±0.89 & 57.23±0.75 & 18.52±0.47 & 39.97±0.65 & 18.55±0.60 \\
          \rowcolor{gray!25} & \drc  & {\textbf{78.39±0.02}} & {\textbf{0.90±0.04}} & {\textbf{3.42±0.03}} & {\textbf{83.76±0.49}} & \underline{18.85±1.71} & {\textbf{3.82±0.24}} & {\textbf{25.86±0.21}} & {\textbf{6.88±0.07}} \\

    \midrule
    \multicolumn{1}{c}{\multirow{9}[0]{*}{\makecell{CIFAR\\-80-20\\(All)}}} & ERM   & 60.77±0.37 & 4.60±0.18 & 13.56±0.34 & 81.95±0.18 & 56.81±0.30 & 23.79±0.39 & 60.44±0.72 & 32.04±0.66 \\
          & PC    & \underline{62.10±0.28} & 4.44±0.05 & 12.75±0.14 & 81.40±0.31 & 55.31±1.48 & 23.77±0.36 & 59.27±0.53 & 29.61±0.47 \\
          & LS    & 62.08±0.43 & 3.90±0.14 & 12.22±0.33 & 85.06±0.37 & 46.26±0.72 & {\textbf{8.31±0.81}} & \underline{48.99±0.67} & {\textbf{17.79±0.26}} \\
          & FLSD  & 59.44±0.36 & 5.04±0.14 & 14.68±0.32 & 82.15±0.42 & 55.51±0.96 & 12.92±0.35 & 54.81±0.49 & 23.31±0.17 \\
          & FL    & 60.30±0.12 & 4.73±0.18 & 13.92±0.21 & 81.40±0.17 & 55.35±0.95 & 18.30±0.12 & 56.81±0.26 & 26.06±0.28 \\
          & IFL   & 61.82±0.46 & \underline{3.62±0.03} & \underline{12.07±0.25} & \underline{86.66±0.31} & \underline{44.96±0.63} & 20.80±0.49 & 56.18±0.62 & 27.76±0.25 \\
          & DFL   & 59.94±0.21 & 4.96±0.04 & 14.34±0.14 & 81.59±0.37 & 55.03±0.50 & 17.24±0.22 & 56.71±0.40 & 26.22±0.15 \\
          & RC    & 60.78±0.46 & 4.96±0.10 & 13.92±0.19 & 80.99±0.35 & 58.50±0.65 & 24.25±0.85 & 61.08±0.75 & 30.90±2.45 \\
          \rowcolor{gray!25} & \drc  & {\textbf{62.74±0.44}} & {\textbf{3.19±0.07}} & {\textbf{11.21±0.24}} & {\textbf{87.41±0.21}} & {\textbf{40.02±0.92}} & \underline{9.53±0.41} & {\textbf{48.76±0.71}} & \underline{19.31±0.13} \\
    \midrule
    \multicolumn{1}{c}{\multirow{9}[0]{*}{\makecell{Image\\NetBG\\(All)}}} & ERM   & 85.79±0.11 & \underline{1.23±0.01} & 2.29±0.01 & 62.29±0.37 & 47.97±0.50 & 4.79±0.41 & 20.30±0.20 & 4.71±0.08 \\
          & PC    & 85.57±0.15 & 1.33±0.07 & 2.43±0.08 & 62.42±1.70 & 49.01±3.92 & 8.62±0.46 & 22.45±0.71 & 6.30±0.44 \\
          & LS    & {\textbf{86.62±0.12}} & 1.34±0.03 & \underline{2.27±0.05} & 62.91±0.28 & \underline{43.95±0.46} & 10.01±0.30 & 19.84±0.09 & 4.92±0.03 \\
          & FLSD  & 85.36±0.16 & 1.43±0.05 & 2.56±0.03 & 61.70±1.63 & 49.68±2.06 & 6.30±0.08 & 21.00±0.04 & 4.81±0.01 \\
          & FL    & 85.67±0.11 & 1.29±0.02 & 2.37±0.03 & 62.62±0.19 & 48.03±0.25 & \underline{1.50±0.11} & \underline{19.76±0.12} & \underline{4.42±0.03} \\
          & IFL   & 85.47±0.05 & 1.26±0.03 & 2.38±0.03 & \underline{62.99±0.48} & 48.07±0.74 & 6.67±0.32 & 21.32±0.11 & 5.21±0.07 \\
          & DFL   & 85.66±0.25 & 1.42±0.07 & 2.51±0.11 & 61.69±0.49 & 48.87±2.11 & 1.62±0.07 & 19.94±0.38 & 4.50±0.09 \\
          & RC    & 85.95±0.18 & 1.70±0.03 & 2.73±0.01 & 58.32±0.63 & 52.66±0.28 & 8.88±0.18 & 22.56±0.29 & 6.44±0.08 \\
         \rowcolor{gray!25} & \drc  & \underline{86.12±0.11} & {\textbf{0.89±0.01}} & {\textbf{1.90±0.02}} & {\textbf{67.74±0.34}} & {\textbf{38.18±0.42}} & {\textbf{1.04±0.17}} & {\textbf{18.71±0.18}} & {\textbf{4.41±0.06}} \\
        \midrule
        \multicolumn{1}{c}{\multirow{9}[0]{*}{\makecell{Image\\NetBG\\(hard set)}}} & ERM   & 81.51±0.15 & \underline{2.09±0.02} & 3.91±0.02 & 62.91±0.36 & 56.87±0.50 & 6.32±0.52 & 26.34±0.30 & 6.12±0.11 \\
          & PC    & 81.24±0.17 & 2.23±0.12 & 4.11±0.14 & 63.03±1.70 & 57.12±2.79 & 11.20±0.57 & 29.14±0.93 & 8.18±0.58 \\
          & LS    & {\textbf{82.60±0.16}} & 2.17±0.04 & \underline{3.78±0.07} & \underline{63.60±0.31} & \underline{51.89±0.65} & 9.87±0.39 & \underline{25.24±0.10} & 6.07±0.03 \\
          & FLSD  & 80.94±0.20 & 2.41±0.07 & 4.36±0.06 & 62.41±1.60 & 57.78±1.66 & 6.32±0.13 & 26.88±0.06 & 6.06±0.02 \\
          & FL    & 81.36±0.13 & 2.18±0.03 & 4.04±0.05 & 63.19±0.16 & 56.30±0.58 & 2.08±0.25 & 25.61±0.15 & \underline{5.71±0.05} \\
          & IFL   & 81.11±0.04 & 2.14±0.05 & 4.05±0.06 & \underline{63.60±0.43} & 57.04±0.74 & 8.69±0.38 & 27.66±0.13 & 6.76±0.08 \\
          & DFL   & 81.34±0.35 & 2.38±0.12 & 4.24±0.19 & 62.33±0.46 & 56.63±1.86 & \underline{1.75±0.36} & 25.79±0.53 & 5.79±0.12 \\
          & RC    & 81.82±0.24 & 2.73±0.03 & 4.49±0.03 & 59.33±0.70 & 59.97±0.51 & 11.52±0.22 & 29.13±0.38 & 8.33±0.10 \\
          \rowcolor{gray!25} & \drc  & \underline{81.97±0.15} & {\textbf{1.51±0.01}} & {\textbf{3.24±0.04}} & {\textbf{68.38±0.33}} & {\textbf{47.09±0.48}} & {\textbf{1.32±0.22}} & {\textbf{24.17±0.23}} & {\textbf{5.68±0.08}} \\

    \midrule
    \multicolumn{1}{c}{\multirow{9}[0]{*}{\makecell{Food\\101\\(ID data)}}} & ERM   & 84.99±0.09 & 1.62±0.01 & 2.80±0.02 & 60.33±0.44 & 52.35±0.68 & 4.82±0.18 & 21.90±0.14 & 5.87±0.03 \\
          & PC    & 85.29±0.16 & 1.63±0.02 & 2.77±0.04 & \underline{59.95±0.32} & 52.59±0.97 & 8.12±0.12 & 22.94±0.29 & 7.15±0.09 \\
          & LS    & 85.04±0.10 & 2.21±0.03 & 3.39±0.04 & 57.59±0.31 & 55.80±0.42 & 10.49±0.07 & 23.35±0.07 & 6.79±0.03 \\
          & FLSD  & 85.00±0.07 & 1.75±0.03 & 2.93±0.03 & 58.00±0.42 & 55.94±1.31 & 3.54±0.04 & 21.82±0.04 & 5.47±0.01 \\
          & FL    & 85.37±0.13 & 1.68±0.02 & 2.81±0.03 & 58.61±0.93 & 53.37±0.80 & \underline{1.32±0.10} & 21.11±0.07 & \underline{5.32±0.01} \\
          & IFL   & \underline{86.50±0.12} & \underline{1.52±0.03} & \underline{2.47±0.05} & 57.81±0.20 & {\textbf{51.32±0.37}} & 7.64±0.16 & 21.29±0.26 & 6.66±0.12 \\
          & DFL   & 85.50±0.08 & 1.63±0.02 & 2.74±0.03 & 58.52±0.27 & 52.95±1.04 & {\textbf{0.80±0.15}} & \underline{20.79±0.10} & {\textbf{5.30±0.01}} \\
          & RC    & 85.42±0.26 & 1.61±0.03 & 2.72±0.07 & 58.67±0.15 & 53.70±0.06 & 6.24±0.33 & 21.93±0.45 & 6.11±0.14 \\
          \rowcolor{gray!25} & \drc  & {\textbf{86.53±0.09}} & {\textbf{1.46±0.02}} & {\textbf{2.41±0.01}} & {\textbf{57.93±0.17}} & \underline{51.61±0.89} & 3.66±0.14 & {\textbf{19.95±0.05}} & 5.44±0.02 \\

    \midrule
    \multicolumn{1}{c}{\multirow{9}[0]{*}{\makecell{Came\\lyon\\(hard set)}}} & ERM   & 85.75±1.32 & 3.41±0.66 & 4.48±0.86 & 39.42±0.39 & 73.35±2.54 & 8.73±1.15 & 22.56±2.32 & 4.41±0.54 \\
          & PC    & 85.16±0.47 & 3.42±0.36 & 4.58±0.43 & 40.42±0.61 & 74.48±1.54 & 11.44±0.43 & 25.41±0.93 & 6.38±0.32 \\
          & LS    & 84.73±1.43 & 4.41±0.71 & 5.65±0.94 & 38.21±0.77 & 75.98±2.02 & 13.21±1.57 & 25.94±1.54 & 4.29±0.18 \\
          & FLSD  & 86.09±0.73 & 3.69±0.11 & 4.71±0.22 & 37.85±1.08 & 73.96±0.76 & 9.12±0.77 & 21.99±0.60 & 3.69±0.08 \\
          & FL    & \underline{86.60±0.77} & 3.36±0.19 & 4.30±0.27 & 38.39±1.64 & 72.46±1.24 & \underline{3.45±0.77} & \underline{19.52±0.78} & \underline{3.23±0.11} \\
          & IFL   & 85.85±0.29 & \underline{2.92±0.16} & \underline{3.97±0.13} & \underline{41.19±1.40} & 72.31±0.57 & 11.21±0.32 & 24.46±0.49 & 6.31±0.07 \\
          & DFL   & 85.75±0.84 & 3.03±0.25 & 4.10±0.36 & {\textbf{41.58±1.13}} & \underline{71.72±1.09} & {\textbf{2.87±0.25}} & 19.74±0.92 & {\textbf{3.19±0.13}} \\
          & RC    & 84.36±1.91 & 4.52±0.64 & 5.83±0.93 & 39.43±2.03 & 75.86±2.24 & 12.34±2.05 & 27.14±3.71 & 7.19±1.32 \\
          \rowcolor{gray!25} & \drc  & {\textbf{87.46±1.56}} & {\textbf{2.83±0.43}} & {\textbf{3.66±0.62}} & 37.96±1.62 & {\textbf{71.18±2.54}} & 6.39±1.35 & {\textbf{19.34±2.38}} & 3.52±0.41 \\

    \bottomrule
    \toprule
    \multirow{2}[2]{*}{Dataset} & \multicolumn{1}{c|}{\multirow{2}[2]{*}{Method}} & \multicolumn{1}{c}{ECE} & \multicolumn{1}{c}{Brier} & \multicolumn{1}{c|}{NLL} & \multirow{2}[2]{*}{Dataset} & \multicolumn{1}{c|}{\multirow{2}[2]{*}{Method}} & \multicolumn{1}{c}{ECE} & \multicolumn{1}{c}{Brier} & \multicolumn{1}{c}{NLL} \\
          & \multicolumn{1}{c|}{} & \multicolumn{1}{c}{(↓)}   & (↓)   & \multicolumn{1}{c|}{(↓)}   &       & \multicolumn{1}{c|}{} & (↓)   & (↓)   & (↓) \\
    \midrule
    \multicolumn{1}{c}{\multirow{9}[0]{*}{\makecell{CIFAR\\-8-2\\(hard set)}}} & ERM   & \multicolumn{1}{c}{55.69±23.39} & 136.37±24.91 & \multicolumn{1}{c|}{67.67±29.01} & \multicolumn{1}{c}{\multirow{9}[0]{*}{\makecell{CIFAR\\-80-20\\(hard set)}}} & ERM   & 64.26±0.39 & 151.27±0.57 & 106.94±1.80 \\
          & PC    & \multicolumn{1}{c}{70.60±10.18} & 151.27±13.78 & \multicolumn{1}{c|}{99.21±42.13} &       & PC    & 64.50±1.19 & 152.05±1.46 & 95.59±2.37 \\
          & LS    & \multicolumn{1}{c}{36.33±11.02} & 110.98±9.95 & \multicolumn{1}{c|}{27.55±2.13} &       & LS    & \underline{29.65±0.92} & \underline{116.05±0.73} & \underline{54.92±0.62} \\
          & FLSD  & \multicolumn{1}{c}{57.10±0.44} & 134.02±0.63 & \multicolumn{1}{c|}{50.15±0.73} &       & FLSD  & 45.59±0.39 & 129.95±0.48 & 76.85±0.56 \\
          & FL    & \multicolumn{1}{c}{67.07±0.65} & 146.57±0.99 & \multicolumn{1}{c|}{66.37±0.74} &       & FL    & 53.16±0.35 & 138.38±0.37 & 86.54±0.65 \\
          & IFL   & \multicolumn{1}{c}{\underline{14.01±4.25}} & \underline{97.20±3.93} & \multicolumn{1}{c|}{\underline{26.11±3.04}} &       & IFL   & 46.29±0.38 & 133.23±0.50 & 79.58±0.86 \\
          & DFL   & \multicolumn{1}{c}{65.82±0.28} & 144.78±0.43 & \multicolumn{1}{c|}{66.31±0.46} &       & DFL   & 52.10±0.26 & 137.08±0.34 & 86.46±0.28 \\
          & RC    & \multicolumn{1}{c}{74.56±1.08} & 156.72±1.46 & \multicolumn{1}{c|}{79.05±2.43} &       & RC    & 65.14±2.38 & 152.38±2.96 & 96.58±8.82 \\
          \rowcolor{gray!25} & \drc  & \multicolumn{1}{c}{{\textbf{10.24±0.44}}} & {\textbf{93.44±0.61}} & \multicolumn{1}{c|}{{\textbf{23.37±0.41}}} &       & \drc  & {\textbf{15.38±0.85}} & {\textbf{109.05±0.49}} & {\textbf{49.97±0.35}}\\
    \bottomrule
    \end{tabular}%
  \label{tab:weak_aug_all}%
\end{table*}%

\begin{table*}[!htbp]
\mycustomsize
  \centering
  \caption{The comparison experimental results on different datasets and different methods under the strong augmentation setting. ↓ and ↑ indicate lower and higher values are better. For better presentation, the best and second-best results are in {\textbf{bold}} and \underline{uderline} respectively. For clarity, NLL values are multiplied by 10. Remaining values are reported as percentages (\%). For datasets with both in-distributional (ID) easy samples and hard test samples, we report the results on all samples and hard samples respectively. We mark whether the test samples are sampled from ID or hard data in the table. \textit{Compared with other methods, \drc achieves excellent performance on different metrics in almost all datasets.}}
    \begin{tabular}{cc|c|cccc|ccc}
    \toprule
    \multirow{2}[2]{*}{\makecell{Dataset}} & \multicolumn{1}{c|}{\multirow{2}[2]{*}{Method}} & \multicolumn{1}{c|}{ACC} & \multicolumn{1}{c}{EAURC} & \multicolumn{1}{c}{AURC} & \multicolumn{1}{c}{AUPR} & \multicolumn{1}{c|}{FPR95\%}  & \multicolumn{1}{c}{ECE} & \multicolumn{1}{c}{Brier} & \multicolumn{1}{c}{NLL} \\
    & \multicolumn{1}{c|}{} & \multicolumn{1}{c|}{(↑)} & (↓)   & \multicolumn{1}{c}{(↓)} & Err (↑) & \multicolumn{1}{c|}{TPR(↓)} & (↓)   & (↓)   & (↓) \\
    \midrule
    \multicolumn{1}{c}{\multirow{9}[0]{*}{\makecell{CIFAR\\-8-2\\(All)}}} & ERM   & 79.35±0.27 & 0.70±0.02 & 2.99±0.07 & 83.77±0.83 & 16.29±1.22 & 6.89±0.78 & 25.86±0.77 & 6.78±0.31 \\
          & PC    & 79.65±0.17 & \underline{0.64±0.03} & \underline{2.87±0.01} & \underline{84.54±1.18} & \underline{12.80±1.53} & \underline{4.94±0.44} & \underline{24.35±0.29} & \underline{6.66±0.20} \\
          & LS    & 79.65±0.10 & 4.57±1.80 & 6.80±1.78 & 75.48±3.08 & 29.96±5.26 & 10.01±0.40 & 28.05±0.89 & 7.62±0.34 \\
          & FLSD  & 78.22±0.18 & 2.33±0.06 & 4.89±0.05 & 70.65±0.83 & 50.43±0.94 & 8.58±0.36 & 32.14±0.14 & 10.54±0.11 \\
          & FL    & 78.96±0.09 & 1.07±0.19 & 3.46±0.20 & 78.62±1.84 & 31.57±6.61 & 9.51±1.84 & 29.34±1.73 & 8.78±1.28 \\
          & IFL   & \underline{79.70±0.09} & 0.70±0.01 & 2.92±0.01 & {\textbf{84.86±0.19}} & 12.85±0.28 & 4.97±0.34 & 24.39±0.22 & 6.82±0.16 \\
          & DFL   & 78.69±0.25 & 1.97±0.13 & 4.42±0.07 & 71.32±1.06 & 48.43±2.15 & 13.25±0.24 & 33.70±0.05 & 12.47±0.38 \\
          & RC    & 75.38±4.42 & 2.03±0.47 & 5.48±1.78 & 76.19±3.88 & 43.03±2.23 & 17.36±1.22 & 40.15±4.98 & 23.41±3.95 \\
          \rowcolor{gray!25} & \drc  & {\textbf{79.97±0.09}} & {\textbf{0.62±0.02}} & {\textbf{2.78±0.03}} & 84.09±0.77 & {\textbf{11.50±0.30}} & {\textbf{3.30±0.32}} & {\textbf{23.14±0.09}} & {\textbf{5.98±0.07}} \\

    \midrule
    \multicolumn{1}{c}{\multirow{9}[0]{*}{\makecell{CIFAR\\-80-20\\(All)}}} & ERM   & 63.68±0.47 & 3.42±0.51 & 11.01±0.70 & 84.17±2.74 & 47.71±7.28 & 17.08±4.87 & 52.00±4.53 & 25.22±6.82 \\
          & PC    & 64.10±0.52 & \underline{3.02±0.18} & 10.42±0.40 & \underline{86.81±0.14} & \underline{40.59±0.98} & 16.09±0.34 & 50.17±0.79 & 21.96±0.67 \\
          & LS    & 63.73±0.61 & 4.17±0.90 & 11.73±0.63 & 83.69±3.16 & 46.24±4.89 & {\textbf{5.07±2.14}} & \underline{47.08±0.38} & \underline{18.16±0.75} \\
          & FLSD  & 61.88±0.35 & 5.02±0.03 & 13.44±0.16 & 80.19±0.35 & 55.50±0.33 & 16.49±0.06 & 54.23±0.36 & 26.09±0.42 \\
          & FL    & 63.06±0.45 & 3.30±0.07 & 11.16±0.28 & 86.46±0.22 & 41.15±0.98 & 9.29±0.30 & 47.72±0.46 & 18.56±0.30 \\
          & IFL   & \underline{64.43±0.62} & 3.02±0.09 & \underline{10.27±0.33} & 86.66±0.48 & 41.42±1.10 & 21.66±1.07 & 54.25±1.46 & 29.84±1.70 \\
          & DFL   & 63.47±0.48 & 3.31±0.10 & 10.99±0.31 & 85.89±0.33 & 41.43±0.72 & \underline{8.42±1.17} & \underline{47.08±0.52} & \underline{18.16±0.35} \\
          & RC    & 60.20±1.04 & 4.97±0.30 & 14.23±0.81 & 81.19±0.12 & 58.94±0.96 & 24.73±2.69 & 62.20±1.83 & 42.31±3.34 \\
          \rowcolor{gray!25} & \drc  & {\textbf{65.90±0.33}} & {\textbf{2.37±0.11}} & {\textbf{8.98±0.24}} & {\textbf{88.38±0.10}} & {\textbf{34.76±0.29}} & 9.84±0.40 & {\textbf{44.59±0.44}} & {\textbf{17.20±0.19}} \\

    \midrule
    \multicolumn{1}{c}{\multirow{9}[0]{*}{\makecell{Image\\NetBG\\(All)}}} & ERM   & 87.35±0.14 & \underline{0.98±0.02} & \underline{1.82±0.04} & 61.85±1.00 & \underline{44.06±1.25} & 3.47±0.10 & 17.79±0.30 & \underline{4.04±0.07} \\
          & PC    & 86.88±0.36 & 1.09±0.06 & 1.99±0.11 & 61.90±0.71 & 45.12±2.37 & 7.19±0.46 & 19.95±0.89 & 5.21±0.35 \\
          & LS    & {\textbf{87.54±0.39}} & 1.26±0.02 & 2.07±0.07 & 60.55±0.61 & 44.85±0.89 & 4.42±0.27 & \underline{17.73±0.47} & 4.20±0.12 \\
          & FLSD  & 86.64±0.34 & 1.21±0.04 & 2.15±0.09 & 61.42±0.94 & 47.22±0.32 & 7.51±0.38 & 19.56±0.33 & 4.50±0.07 \\
          & FL    & 86.70±0.20 & 1.17±0.07 & 2.10±0.08 & 62.08±1.79 & 46.23±2.16 & 3.92±0.13 & 18.65±0.37 & 4.21±0.09 \\
          & IFL   & 87.26±0.24 & 1.03±0.04 & 1.88±0.06 & \underline{62.24±0.77} & 45.32±1.56 & 4.89±0.27 & 18.37±0.31 & 4.32±0.11 \\
          & DFL   & 87.11±0.48 & 1.17±0.06 & 2.04±0.13 & 60.91±1.27 & 46.51±0.79 & \underline{2.96±0.26} & 18.11±0.66 & 4.09±0.15 \\
          & RC    & 87.21±0.29 & 1.23±0.05 & 2.08±0.08 & 59.35±1.32 & 48.24±1.78 & 6.22±0.26 & 19.22±0.33 & 4.76±0.09 \\
          \rowcolor{gray!25} & \drc  & \underline{87.53±0.42} & {\textbf{0.94±0.04}} & {\textbf{1.75±0.09}} & {\textbf{62.53±0.60}} & {\textbf{42.47±0.65}} & {\textbf{2.23±0.42}} & {\textbf{17.45±0.51}} & {\textbf{4.01±0.14}} \\

    \midrule
    \multicolumn{1}{c}{\multirow{9}[0]{*}{\makecell{Image\\NetBG\\(hard set)}}} & ERM   & 83.52±0.15 & \underline{1.68±0.03} & \underline{3.12±0.06} & 62.37±1.06 & \underline{53.02±1.13} & 4.61±0.10 & 23.13±0.37 & 5.25±0.10 \\
          & PC    & 82.91±0.47 & 1.83±0.11 & 3.38±0.19 & 62.65±0.57 & 54.36±2.34 & 9.40±0.63 & 25.93±1.17 & 6.77±0.46 \\
          & LS    & \underline{83.81±0.52} & 2.07±0.03 & 3.46±0.12 & 61.05±0.67 & 53.05±0.90 & 4.21±0.30 & \underline{22.86±0.62} & 5.31±0.16 \\
          & FLSD  & 82.63±0.43 & 2.05±0.08 & 3.66±0.16 & 61.96±0.92 & 54.95±0.86 & 7.92±0.44 & 24.99±0.44 & 5.67±0.10 \\
          & FL    & 82.72±0.25 & 1.99±0.11 & 3.57±0.14 & 62.59±1.83 & 54.38±2.00 & 4.04±0.17 & 24.01±0.48 & 5.38±0.12 \\
          & IFL   & 83.41±0.35 & 1.73±0.06 & 3.19±0.11 & \underline{62.87±0.77} & 53.79±1.17 & 6.43±0.37 & 23.87±0.47 & 5.61±0.16 \\
          & DFL   & 83.28±0.64 & 1.97±0.10 & 3.46±0.21 & 61.44±1.35 & 54.90±0.88 & \underline{3.03±0.35} & 23.36±0.86 & \underline{5.24±0.21} \\
          & RC    & 83.41±0.37 & 2.04±0.09 & 3.51±0.14 & 59.99±1.34 & 56.03±1.06 & 8.14±0.34 & 24.93±0.41  & 6.18±0.12 \\
          \rowcolor{gray!25} & \drc  & {\textbf{83.83±0.55}} & {\textbf{1.58±0.05}} & {\textbf{2.97±0.15}} & {\textbf{63.01±0.62}} & {\textbf{51.31±0.14}} & {\textbf{2.97±0.56}} & {\textbf{22.59±0.67}} & {\textbf{5.18±0.18}} \\

    \midrule
    \multicolumn{1}{c}{\multirow{9}[0]{*}{\makecell{Food\\101\\(ID)}}} & ERM   & 87.26±0.14 & 1.25±0.02 & 2.10±0.02 & 58.21±1.04 & 50.47±1.13 & 2.37±0.08 & \underline{18.44±0.12} & \underline{4.68±0.04} \\
          & PC    & \underline{87.54±0.08} & \underline{1.24±0.01} & \underline{2.05±0.02} & 56.87±0.23 & 50.58±0.30 & 5.87±0.12 & 19.17±0.11 & 5.45±0.02 \\
          & LS    & 87.33±0.03 & 1.80±0.03 & 2.64±0.03 & 53.76±0.22 & 54.59±0.56 & 19.70±0.16 & 23.58±0.03 & 6.91±0.01 \\
          & FLSD  & 85.98±0.09 & 1.55±0.01 & 2.58±0.01 & 57.34±0.55 & 54.40±0.37 & 6.04±0.04 & 20.90±0.08 & 5.21±0.02 \\
          & FL    & 86.32±0.22 & 1.45±0.02 & 2.43±0.05 & \underline{58.26±0.43} & 51.60±0.45 & 3.62±0.12 & 19.91±0.17 & 4.95±0.04 \\
          & IFL   & {\textbf{87.67±0.07}} & {\textbf{1.22±0.04}} & {\textbf{2.01±0.03}} & 57.87±0.80 & {\textbf{49.71±1.12}} & 5.05±0.07 & 18.58±0.07 & 5.10±0.01 \\
          & DFL   & 86.80±0.11 & 1.40±0.04 & 2.31±0.06 & 57.31±0.47 & 52.08±1.04 & {\textbf{1.29±0.13}} & 19.15±0.23 & 4.80±0.06 \\
          & RC    & 87.23±0.01 & 1.26±0.03 & 2.12±0.03 & 57.80±1.23 & 50.83±1.60 & 3.04±0.08 & 18.65±0.12 & 4.79±0.05 \\
          \rowcolor{gray!25} & \drc  & 87.32±0.12 & 1.25±0.01 & 2.09±0.03 & {\textbf{58.27±0.42}} & \underline{50.33±0.76} & \underline{2.17±0.13} & {\textbf{18.34±0.13}} & {\textbf{4.67±0.02}} \\

    \midrule          
    \multicolumn{1}{c}{\multirow{9}[0]{*}{\makecell{Came\\lyon\\(hard set)}}} & ERM   & 90.21±0.38 & 1.41±0.06 & 1.91±0.10 & 38.69±0.55 & 63.63±0.88 & \underline{4.60±0.19} & \underline{14.89±0.51} & \underline{2.59±0.08} \\
          & PC    & 87.30±0.40 & 1.97±0.10 & 2.81±0.14 & {\textbf{40.47±0.75}} & 70.33±0.54 & 9.91±0.80 & 21.92±1.14 & 5.09±0.54 \\
          & LS    & \underline{92.52±0.48} & 1.28±0.16 & 1.57±0.19 & 34.84±0.71 & \underline{58.96±1.14} & 18.69±0.36 & 18.66±0.47 & 3.49±0.06 \\
          & FLSD  & 92.24±0.28 & 1.12±0.05 & 1.43±0.07 & 35.29±0.31 & 59.69±1.19 & 12.81±0.27 & 15.83±0.17 & 2.85±0.03 \\
          & FL    & 92.30±0.24 & \underline{1.11±0.04} & \underline{1.41±0.06} & 35.17±0.12 & 59.67±1.09 & 12.91±0.28 & 15.82±0.10 & 2.85±0.02 \\
          & IFL   & 88.99±0.85 & 1.59±0.19 & 2.23±0.29 & \underline{40.02±0.69} & 65.91±1.59 & 7.15±0.92 & 17.71±1.57 & 3.36±0.33 \\
          & DFL   & 90.60±1.30 & 1.72±0.21 & 2.19±0.34 & 36.44±2.74 & 61.35±0.52 & 7.18±0.96 & 15.39±0.93 & 2.66±0.11 \\
          & RC    & 91.55±0.58 & 2.62±0.33 & 2.99±0.38 & 31.57±1.27 & 64.19±2.47 & 7.10±0.60 & 15.28±1.16 & 5.31±0.58 \\
          \rowcolor{gray!25} & \drc  & {\textbf{93.43±0.22}} & {\textbf{0.83±0.06}} & {\textbf{1.05±0.07}} & 34.34±0.92 & {\textbf{57.07±1.58}} & {\textbf{2.30±0.38}} & {\textbf{10.16±0.39}} & {\textbf{1.88±0.10}} \\

    \bottomrule
    \toprule
    \multirow{2}[2]{*}{Dataset} & \multicolumn{1}{c|}{\multirow{2}[2]{*}{Method}} & \multicolumn{1}{c}{ECE} & \multicolumn{1}{c}{Brier} & \multicolumn{1}{c|}{NLL} & \multirow{2}[2]{*}{Dataset} & \multicolumn{1}{c|}{\multirow{2}[2]{*}{Method}} & \multicolumn{1}{c}{ECE} & \multicolumn{1}{c}{Brier} & \multicolumn{1}{c}{NLL} \\
          & \multicolumn{1}{c|}{} & \multicolumn{1}{c}{(↓)}   & (↓)   & \multicolumn{1}{c|}{(↓)}   &       & \multicolumn{1}{c|}{} & (↓)   & (↓)   & (↓) \\
    \midrule
    \multicolumn{1}{c}{\multirow{9}[0]{*}{\makecell{CIFAR\\-8-2\\(hard set)}}} & ERM   & \multicolumn{1}{c}{26.42±2.31} & 103.12±1.77 & \multicolumn{1}{c|}{27.70±1.06} & \multicolumn{1}{c}{\multirow{9}[0]{*}{\makecell{CIFAR\\-80-20\\(hard set)}}} & ERM   & 47.87±13.87 & 134.26±14.91 & 87.04±27.59 \\
          & PC    & \multicolumn{1}{c}{15.17±1.88} & \underline{96.69±1.59} & \multicolumn{1}{c|}{25.52±0.86} &       & PC    & 37.96±0.75 & 124.76±0.48 & 67.83±1.09 \\
          & LS    & \multicolumn{1}{c}{39.34±4.34} & 113.79±4.10 & \multicolumn{1}{c|}{29.97±1.45} &       & LS    & 30.26±3.68 & 116.36±2.94 & \underline{58.00±2.39} \\
          & FLSD  & \multicolumn{1}{c}{52.43±0.34} & 127.95±0.66 & \multicolumn{1}{c|}{45.84±0.71} &       & FLSD  & 55.18±0.87 & 140.49±0.93 & 94.46±1.79 \\
          & FL    & \multicolumn{1}{c}{43.23±7.89} & 118.76±8.28 & \multicolumn{1}{c|}{37.93±6.23} &       & FL    & 30.87±1.14 & 116.85±0.95 & 60.91±0.90 \\
          & IFL   & \multicolumn{1}{c}{\underline{14.40±1.39}} & 96.78±0.89 & \multicolumn{1}{c|}{\underline{25.34±0.60}} &       & IFL   & 54.86±2.30 & 141.27±2.32 & 93.67±4.71 \\
          & DFL   & \multicolumn{1}{c}{60.51±1.02} & 138.24±1.58 & \multicolumn{1}{c|}{55.67±2.37} &       & DFL   & \underline{29.68±2.15} & \underline{116.09±1.92} & 59.39±1.78 \\
          & RC    & \multicolumn{1}{c}{58.97±12.40} & 137.18±14.39 & \multicolumn{1}{c|}{74.23±24.88} &       & RC    & 65.26±6.81 & 152.99±7.96 & 128.48±14.80 \\
          \rowcolor{gray!25} & \drc  & \multicolumn{1}{c}{{\textbf{8.81±1.51}}} & {\textbf{92.69±0.66}} & \multicolumn{1}{c|}{{\textbf{22.93±0.35}}} &       & \drc  & {\textbf{20.57±0.67}} & {\textbf{112.30±0.61}} & {\textbf{51.92±0.33}} \\
    \bottomrule
    \end{tabular}%
  \label{tab:strong_aug_all}%
\end{table*}%

\begin{figure*}[!ht] 
	\centering  
	\subfigure{
		\includegraphics[width=0.22\linewidth]{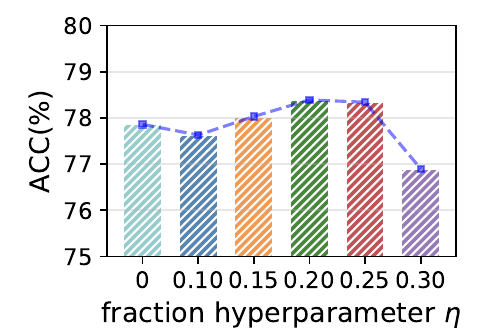}} 
	\subfigure{
		\label{level.sub.2}
		\includegraphics[width=0.22\linewidth]{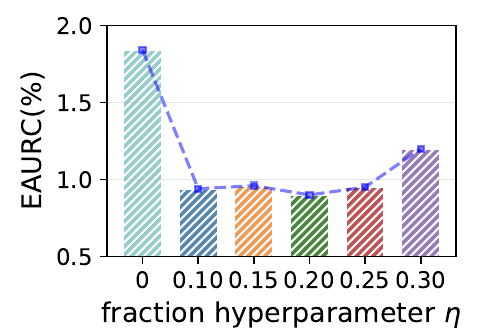}} 
	\subfigure{
		\label{level.sub.3}
		\includegraphics[width=0.22\linewidth]{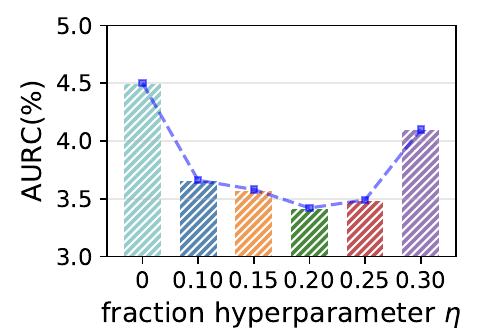}} 
	\subfigure{
		\label{level.sub.4}
		\includegraphics[width=0.22\linewidth]{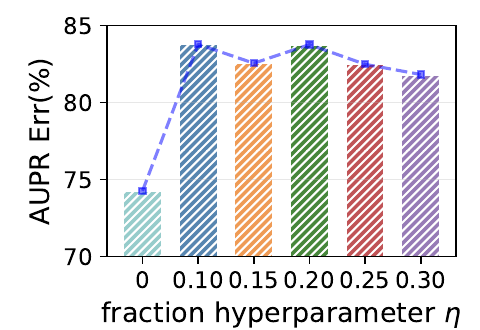}} 
 	\subfigure{
		\label{level.sub.5}
		\includegraphics[width=0.22\linewidth]{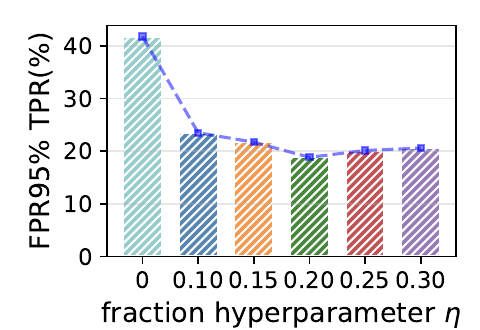}} 
   	\subfigure{
		\label{level.sub.6}
		\includegraphics[width=0.22\linewidth]{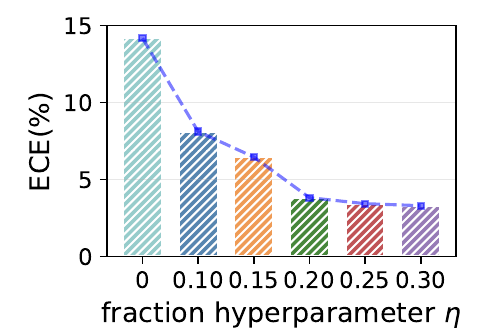}} 
       \subfigure{
		\label{level.sub.7}
		\includegraphics[width=0.22\linewidth]{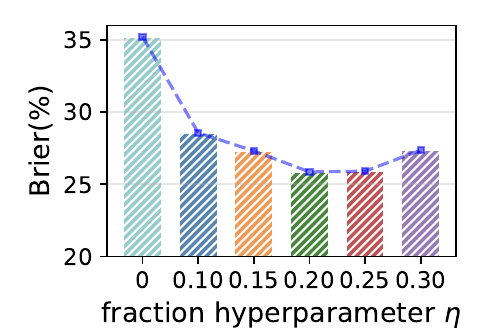}} 
       \subfigure{
		\label{level.sub.8}
		\includegraphics[width=0.22\linewidth]{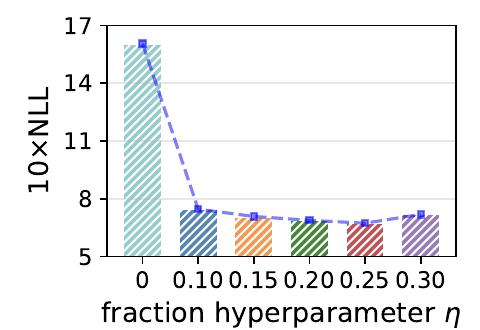}} 
	\caption{Performance of the model on multiple metrics with varying fraction hyperparameter $\eta$, while fixing the regularization strength $\beta$ to 1, on the CIFAR-8-2 dataset.}
 \label{fig:sup:cifar10:eta}
\end{figure*}

\begin{figure*}[!htbp] 
	\centering  
	\subfigure{
		\includegraphics[width=0.22\linewidth]{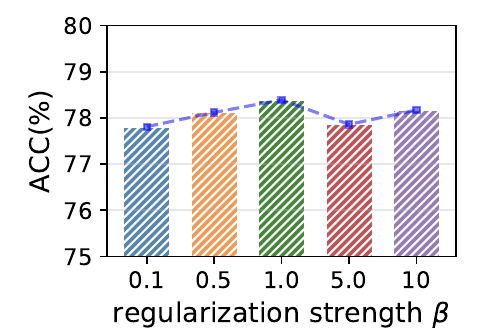}} 
	\subfigure{
		\label{level.sub.2}
		\includegraphics[width=0.22\linewidth]{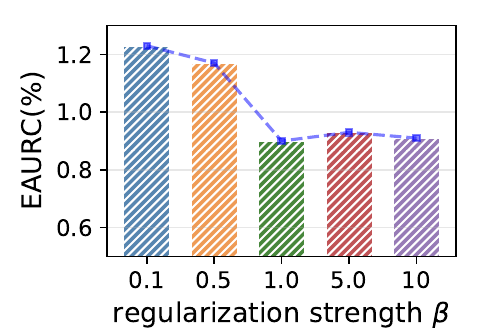}} 
	\subfigure{
		\label{level.sub.3}
		\includegraphics[width=0.22\linewidth]{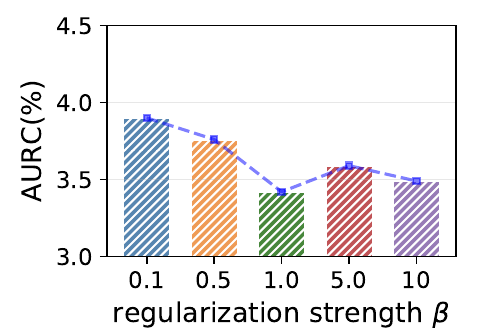}} 
	\subfigure{
		\label{level.sub.4}
		\includegraphics[width=0.22\linewidth]{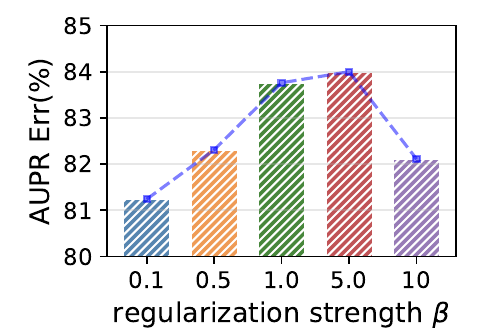}} 
 	\subfigure{
		\label{level.sub.5}
		\includegraphics[width=0.22\linewidth]{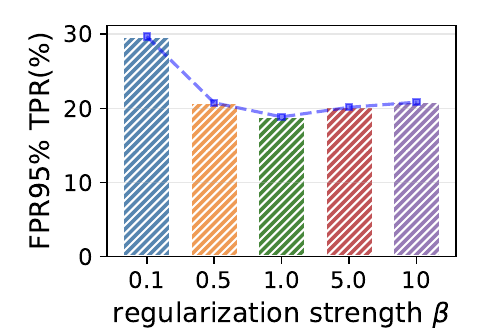}} 
   	\subfigure{
		\label{level.sub.6}
		\includegraphics[width=0.22\linewidth]{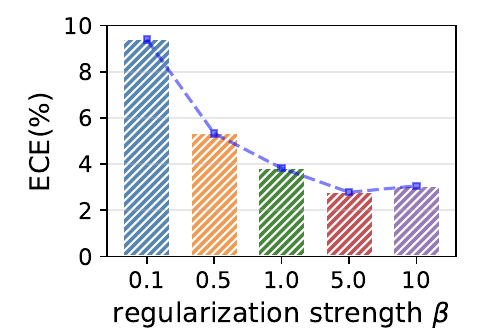}} 
       \subfigure{
		\label{level.sub.7}
		\includegraphics[width=0.22\linewidth]{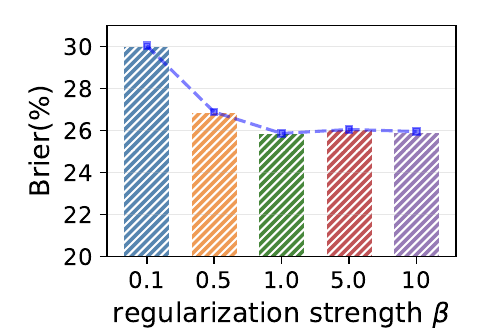}} 
       \subfigure{
		\label{level.sub.8}
		\includegraphics[width=0.22\linewidth]{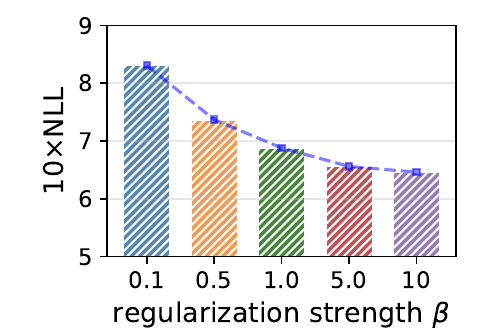}} 
	\caption{Performance of the model on different metrics with varying regularization strength hyperparameter $\beta$, while fixing the  outlier fraction hyperparameter $\eta$ to 0.2, on the CIFAR-8-2 dataset.}
 \label{fig:sup:cifar10:beta}
\end{figure*}
\textbf{Datasets.} We conduct extensive experiments on multiple datasets with potential challenging samples in the training or test dataset, including \texttt{CIFAR-8-2}, \texttt{CIFAR-80-20} \cite{krizhevsky2009learning}, \texttt{Food101} \cite{bossard2014food}, \texttt{Camelyon17} \cite{bandi2018detection,koh2021wilds}, and \texttt{ImageNetBG} \cite{xiao2020noise}. The datasets used in the experiments are described in detail here.
\begin{itemize}
\item \texttt{CIFAR-8-2}: The CIFAR-8-2 dataset is artificially constructed to evaluate the performance of the model when the fraction of challenging samples in the dataset is available. Specifically, we randomly select 8 classes from CIFAR10 \cite{krizhevsky2009learning} as in-distribution samples, while the remaining samples are randomly relabeled as one of the selected classes to serve as challenging samples beyond capability of the model. Then the  fraction $\eta$ of the CIFAR-8-2 dataset is about 20\%. 
\item \texttt{CIFAR-80-20}: Samilar as CIFAR-8-2, we randomly select 80 classes from CIFAR100 \cite{krizhevsky2009learning} dataset, and relabel other samples as one of the selected classes to serve as challenging samples (outliers). The true challenging samples fraction $\eta$ of the CIFAR-80-20 dataset is about 20\%. 
\item \texttt{Camelyon17}: Camelyon17 is a pathology image dataset containing over 450,000 lymph node scans from 5 different hospitals, used for detecting cancerous tissues in images \cite{bandi2018detection}. Similar to previous work \cite{koh2021wilds}, we take part of the data from 3 hospitals as the training set. The remaining data from these 3 hospitals, together with data from another hospital, are used as the validation set. The last hospital is used as a challenging test set with distributional shift samples. Notably, due to differences in pathology staining methods between hospitals, even data within the same hospital can be seen as sampled from multiple subpopulations. We verify on the Camelyon17 dataset whether models can achieve more robust generalization performance when the training set contains multiple subgroups.
\item \texttt{ImageNetBG}: ImageNetBG is a benchmark dataset for evaluating the dependence of classifiers on image backgrounds \cite{xiao2020noise}. It consists of a 9-class subset of ImageNet (ImageNet-9) and provides bounding boxes that allow removing the background. Similar as in previous settings \cite{yang2023change}, we train models on the original IN-9L (with background) set, adjust hyperparameters based on validation accuracy, and evaluate not only on the test set (in-distribution data), but also on the MIXED-RAND, NO-FG and ONLY-FG test set (challenging samples).
\item \texttt{Food101}: Food101 is a commonly used food classification dataset containing 101 food categories with a total of 101,000 images \cite{bossard2014food}. For each category, there are 750 training images and 250 manually verified test images. The training images are intentionally unclean and contain some amount of noise, primarily in the form of intense colors and occasionally wrong labels, which can be seen as challenging samples in the training dataset. 
\end{itemize}

\textbf{Evaluation metrics.} Following the standard metrics used in previous works \cite{moon2020confidence, corbiere2019addressing}, we evaluate the performance from three perspectives: (1) the accuracy (ACC) on the test set; (2) ordinal ranking based confidence evaluation, including area under the risk coverage curve (AURC), excess-AURC (EAURC) \cite{geifman2018bias}, false positive rate when the true positive rate is 95\% (FPR95\%TPR), and area under precision-recall curve with incorrectly classified examples as the positive class (AUPRErr) \cite{corbiere2019addressing}; (3) calibration-based confidence evaluation, including expected calibration error (ECE) \cite{guo2017calibration}, the Brier score (Brier) \cite{brier1950verification} and negative log likelihood (NLL). For datasets where the test set consists of a mixture of in-distribution and challenging samples, we show both the performance on all test samples and the performance on only challenging samples. For the CIFAR-8-2 and CIFAR-80-20 datasets, since the labels of the challenging samples are set randomly, we only report the calibration-based confidence evaluation metrics. 
The evaluation metrics used in the experiments are described in detail here.
\begin{itemize}
\item \texttt{AURC} and \texttt{EAURC}: The AURC is defined as the area under the risk-coverage curve \cite{geifman2017selective}, where risk represents the error rate and coverage refers to the proportion of samples with confidence estimates exceeding a specified confidence threshold. A lower AURC indicates that correct and incorrect samples can be effectively separated based on the confidence of the samples. However, AURC is influenced by the predictive performance of the model. To allow for meaningful comparisons across models with different performance, Excess-AURC (E-AURC) is proposed by \cite{geifman2018bias} by subtracting the optimal AURC (the minimum possible value for a given model) from the empirical AURC.
\item \texttt{AUPRErr}: AUPRErr represents the area under the precision-recall curve where misclassified samples (i.e., incorrect predictions) are used as positive examples. This metric can evaluate the capability of the error detector to distinguish between incorrect and correct samples. A higher AUPRErr usually indicates better error detection performance \cite{corbiere2019addressing}.
\item  \texttt{FPR95\%TPR}: The FPR95\%TPR metric measures the false positive rate (FPR) when the true positive rate (TPR) is fixed at 95\%. This metric can be interpreted as the probability that an incorrect prediction is mistakenly categorized as a correct prediction, when the TPR is set to 95\%.
\item \texttt{ECE}: The Expected Calibration Error (ECE) provides a measure of the alignment between the predicted confidence scores and labels. It partitions the confidence sores into multiple equally spaced intervals, computes the difference between accuracy and average confidence in each interval, and then aggregates the results weighted by the number of samples. Lower ECE usually indicates better calibration.
\item \texttt{NLL}: The Negative Log Likelihood (NLL) measures the log loss between the predicted probabilities and the one-hot label encodings. Lower NLL corresponds to higher likelihood of the predictions fitting the true distribution.
\item \texttt{Brier}: The Brier score calculates the mean squared error between the predicted probabilities and the one-hot label.
\end{itemize}

\begin{figure*}[!ht] 
	\centering  
	\subfigure{
		\includegraphics[width=0.22\linewidth]{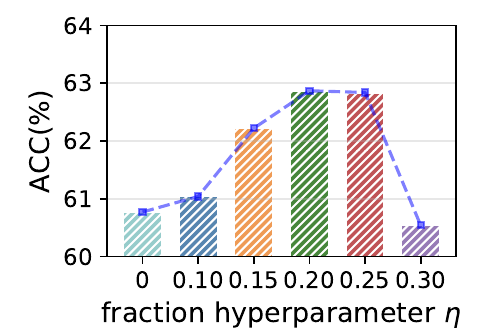}} 
	\subfigure{
		\label{level.sub.2}
		\includegraphics[width=0.22\linewidth]{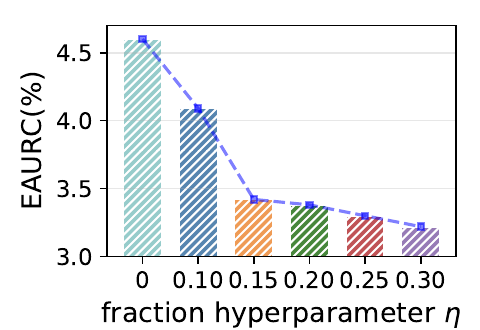}} 
	\subfigure{
		\label{level.sub.3}
		\includegraphics[width=0.22\linewidth]{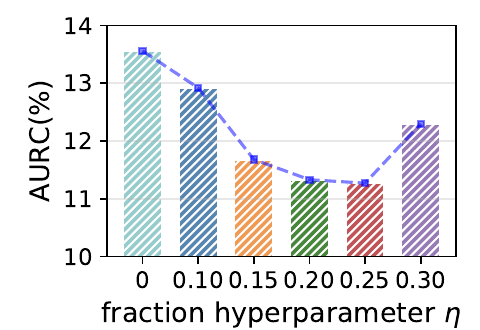}} 
	\subfigure{
		\label{level.sub.4}
		\includegraphics[width=0.22\linewidth]{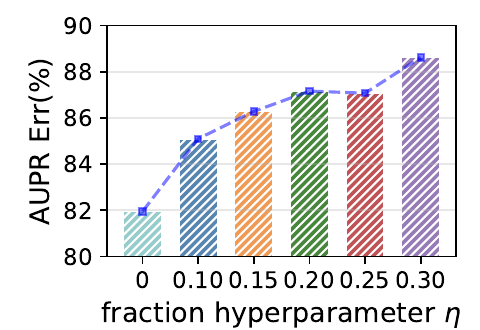}} 
 	\subfigure{
		\label{level.sub.5}
		\includegraphics[width=0.22\linewidth]{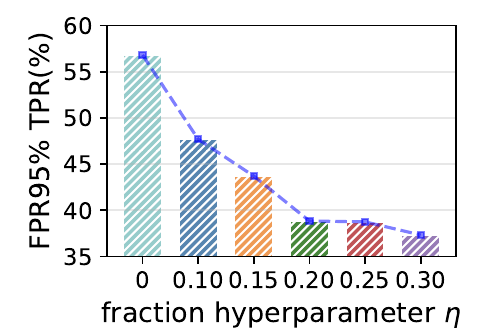}} 
   	\subfigure{
		\label{level.sub.6}
		\includegraphics[width=0.22\linewidth]{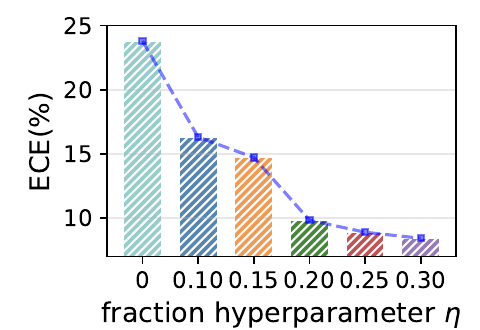}} 
       \subfigure{
		\label{level.sub.7}
		\includegraphics[width=0.22\linewidth]{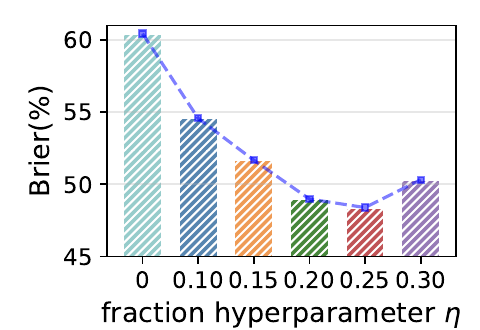}} 
       \subfigure{
		\label{level.sub.8}
		\includegraphics[width=0.22\linewidth]{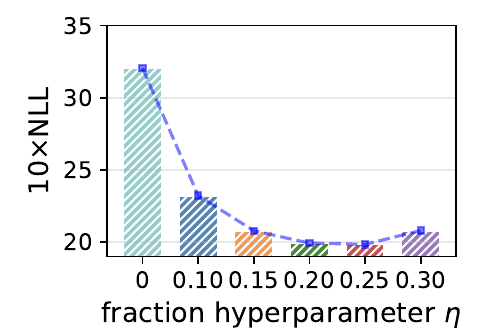}} 
	\caption{Performance of the model on multiple metrics with varying challenging samples fraction hyperparameter $\eta$, while fixing the regularization strength $\beta$ to 1, on the CIFAR-80-20 dataset.}
  \label{fig:sup:cifar100:eta}
\end{figure*}

\begin{figure*}[!htbp] 
	\centering  
	\subfigure{
		\includegraphics[width=0.22\linewidth]{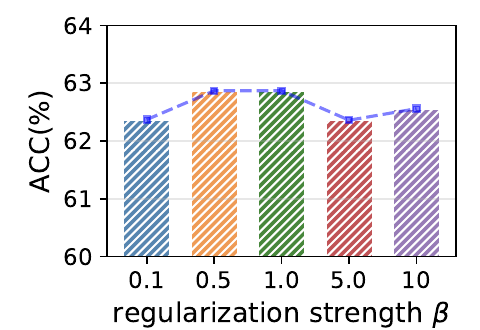}} 
	\subfigure{
		\label{level.sub.2}
		\includegraphics[width=0.22\linewidth]{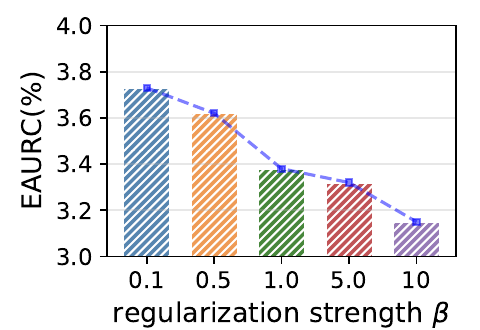}} 
	\subfigure{
		\label{level.sub.3}
		\includegraphics[width=0.22\linewidth]{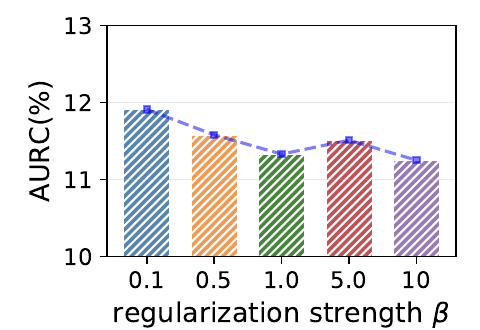}} 
	\subfigure{
		\label{level.sub.4}
		\includegraphics[width=0.22\linewidth]{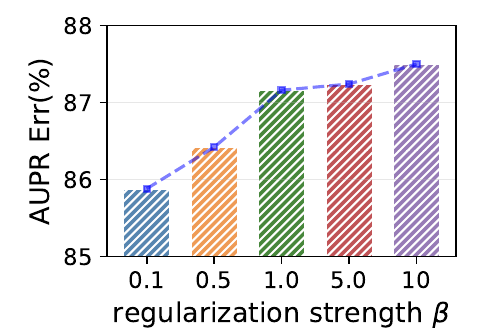}} 
 	\subfigure{
		\label{level.sub.5}
		\includegraphics[width=0.22\linewidth]{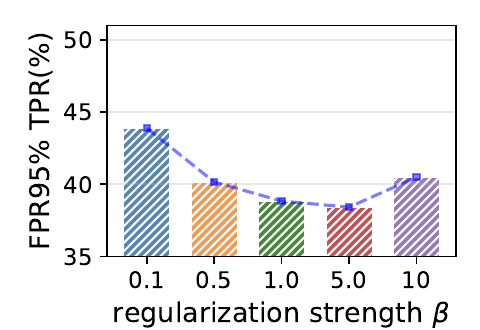}} 
   	\subfigure{
		\label{level.sub.6}
		\includegraphics[width=0.22\linewidth]{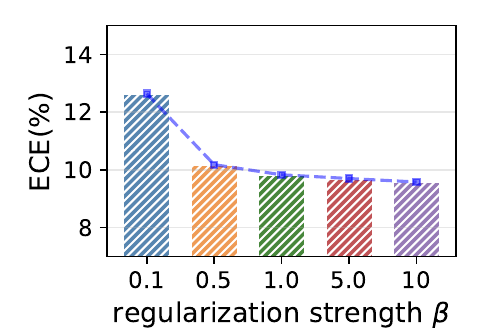}} 
       \subfigure{
		\label{level.sub.7}
		\includegraphics[width=0.22\linewidth]{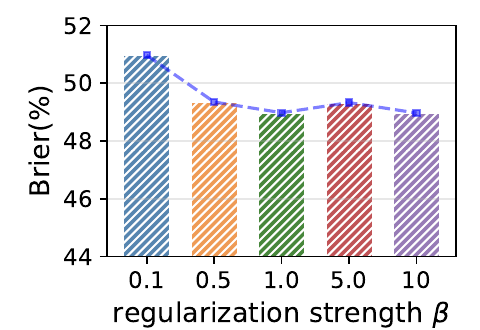}} 
       \subfigure{
		\label{level.sub.8}
		\includegraphics[width=0.22\linewidth]{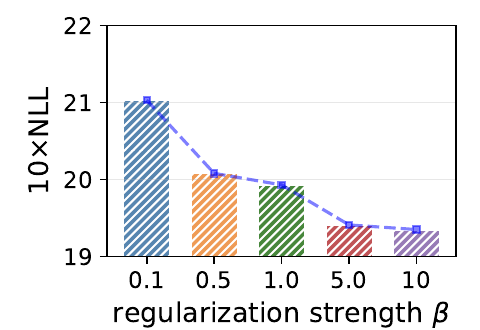}} 
	\caption{Performance of the model on different metrics with varying regularization strength hyperparameter $\beta$, while fixing the challenging samples fraction hyperparameter $\eta$ to 0.2, on the CIFAR-80-20 dataset.}
   \label{fig:sup:cifar100:beta}
\end{figure*}
\textbf{Experimental settings.}  
Since data augmentation has been widely applied during the training of deep neural networks and has achieved remarkable performance, we perform experiments under two different settings including standard setting and strong augmentation setting to validate the effectiveness of \drc. In this way, we can fully explore how challenging samples introduced by augmentation techniques affect the calibration methods. Specifically, under the standard setting, same as standard setup of deep neural network training \cite{he2016deep, zagoruyko2016wide}, we employ basic augmentation methods such as random cropping and image flipping. Moreover,  under the strong augmentation setting, we use random augmentation \cite{cubuk2020randaugment} and cutout \cite{devries2017improved} augmentation methods. For all methods, under the exactly same setting, we tune the hyperparameters based on the accuracy of validation set and run over three times to report the means and standard deviations. For the CIFAR-8-2 and CIFAR-80-20 datasets, we use a randomly initialized WideResNet-28-10 \cite{zagoruyko2016wide} as the backbone network; for the Camelyon17 dataset, we use a DenseNet-121 \cite{huang2017densely} network pre-trained on ImageNet as the backbone network; for other datasets, we use a ResNet-50 \cite{he2016deep} network pre-trained on ImageNet as the backbone network. For all methods, we adjust the hyperparameters according to the performance on the validation set to obtain the best accuracy, and report other corresponding metrics.

\textbf{Comparison methods.} We conduct comparative experiments with multiple baseline methods, including empirical risk minimization (ERM), penalizing confidence (PC) \cite{pereyra2017regularizing}, label smoothing (LS) \cite{muller2019does}, focal loss (FL) \cite{mukhoti2020calibrating}, sample dependent focal loss (FLSD) \cite{mukhoti2020calibrating}, 
inverse focal loss (IFL) \cite{wang2021rethinking}, 
dual focal loss \cite{tao2023dual}. Besides we conduct ablation study by removing challenging samples with high probability (RC) during training. The Comparison methods are described in detail here.
\begin{itemize}
    \item \texttt{ERM} trains the model by minimizing empirical risk on the training data, utilizing cross-entropy as the loss function.
    \item \texttt{PC} trains the model with cross-entropy loss while regularizing the neural networks by penalizing low-entropy predictions.
    \item \texttt{LS} is a regularization technique that trains the neural network with softened target labels.
    \item \texttt{FLSD} refers to a sample-dependent focal loss, where the hyperparameters of the focal loss are set differently for samples with varying confidence scores \cite{mukhoti2020calibrating}.
    \item  \texttt{FL} refers to focal loss, which implicitly regularizes the deep neural network by increasing the weight of samples with larger losses.
    \item \texttt{IFL} implements a simple modification to the weighting term of the original focal loss by assigning larger weights to samples with higher output confidences.
    \item \texttt{DFL} aims to achieve a better balance between over-confidence and under-confidence by maximizing the gap between the ground truth logit and the highest-ranked logit after the ground truth logit.
    \item \texttt{RC}: We conduct ablation studies by removing potential challenging samples beyond the capability of the model during training. Specifically, during training, given $B$ samples, we directly drop the top $\eta B$ samples with the highest losses, where $\eta$ is the predefined fraction of challenging samples.
\end{itemize}

\subsection{Experimental Results}
We conduct experiments to answer the above-posed questions. The main experimental results under standard and strong augmentation settings are presented in the Tab.~\ref{tab:weak_aug_all} and Tab.~\ref{tab:strong_aug_all} respectively.

\textbf{Q1 Effectiveness.} Compared to other methods, \drc achieves superior performance in terms of accuracy. Specifically, as shown in the experimental results, compared with previous methods, \drc achieves top two accuracy rankings consistently on almost all datasets under the standard and strong augmentation settings. 
For example, on the Camelyon17 dataset, \drc achieves the best accuracy performance of 87.46\% and 93.43\% under standard and strong augmentation settings. This improvement can be attributed to the ability of \drc to prevent the neural network from overfitting to challenging samples in the training data, thereby improving the generalization and classification accuracy on the test data slightly. Although \drc demonstrates superior performance, please note that the main goal of calibration methods is to calibrate the model thereby improving trustworthiness, rather than to simply enhance accuracy.

\textbf{Q2 Reliability.} We can obtain the following observations from the experimental results. (1) \drc can obtain the state-of-the-art confidence quality in terms of ranking-based metrics (e.g., EAURC) across almost all datasets. Specifically, under the standard setting, \drc achieves superior EAURC performance on all datasets.  For example, on the CIFAR-80-20 dataset, \drc outperforms the second best method by 0.43\% in terms of EAURC. Moreover, under the strong-augmentation setting, \drc achieves the best performance across all datasets except Food101 dataset. (2) \drc also demonstrates outstanding performance on calibration-based metrics (e.g., ECE). For example, under standard setting, \drc achieves 3.82\% performance on the full test set of CIFAR-8-2 dataset, which is 1.82\% lower than the second best method. The key reason behind the performance improvements is that \drc effectively leverages challenging samples in the training set to provide more explicit confidence supervision to improve the confidence quality, resulting in more reliable predictions.

\textbf{Q3 Robustness.} To evaluate the robustness of \drc, we also show the performance of the different methods on the challenging test dataset, which may exceed the capabilities of the model due to distribution shifts or inherent difficulties in classification. We can draw the following observations. (1) When evaluated on challenging datasets, the performance usually decreases. For example, on the challenging subset of the ImageNetBG test set, the performance is lower across all metrics compared to the full test set. This highlights the necessity to study robust calibration methods. (2) \drc shows excellent performance on challenging test datasets. For example, under the standard setting, \drc achieves 10.24\% and 15.38\% in terms of ECE on the challenging subset of CIFAR-8-2 and CIFAR-80-20, outperforming the second best by 3.77\% and 14.27\% respectively. This indicates that \drc is more robust to potential challenging data, because it effectively utilizes the challenging data during training to impose clear guidance about what should be unknown for the model.

\textbf{Q4 Ablation study.} As shown in Tab.~\ref{tab:weak_aug_all} and Tab.~\ref{tab:strong_aug_all}, compared with ERM and removing challenging samples (RC), \drc consistently shows better performance. Specifically, our experiments demonstrate that utilizing challenging samples for regularization during deep neural network training through \drc can lead to more reliable confidence estimation than ERM and RC. 

\textbf{Q5 Hyperparameter Analysis.} We present the detailed hyperparameter analysis results on the CIFAR-8-2 and CIFAR-80-20 datasets in Fig.\ref{fig:sup:cifar10:eta}, Fig.\ref{fig:sup:cifar10:beta}, Fig.\ref{fig:sup:cifar100:eta} and Fig.\ref{fig:sup:cifar100:beta}. 
Specifically, to evaluate the effect of the challenging samples fraction hyperparameter $\eta$ and regularization strength $\beta$ on the model, we tune one hyperparameter while fixing the other. 
From the experimental results we can draw the following conclusions: (1) As shown in Fig.\ref{fig:sup:cifar10:eta} and Fig.\ref{fig:sup:cifar100:eta}, when the set fraction hyperparameter $\eta$ is close to the true challenging  samples ratio, the model can achieve relatively optimal performance. Meanwhile increasing $\eta$ within a certain range does not significantly degrade the model performance. For example, on most metrics of the CIFAR-8-2 and CIFAR-80-20 datasets, the relatively best performance is achieved at $\eta=0.2$. (2) As shown in Fig.\ref{fig:sup:cifar10:beta} and Fig.\ref{fig:sup:cifar100:beta}, we can find that increasing the regularization strength $\beta$ to a certain level yields relatively good performance, after which further increases does not significantly improve the results. For instance, when $\beta$ exceeds 1, the performance of the model remains relatively stable, with most metrics changing slightly.